\documentclass[11pt]{article}

\usepackage[T1]{fontenc}
\usepackage[utf8]{inputenc}
\usepackage[letterpaper,margin=1in]{geometry}
\usepackage[numbers,sort&compress]{natbib}

\usepackage{amsmath, amssymb, amsfonts, mathtools, amsthm}
\usepackage{nicefrac}
\usepackage{microtype}

\usepackage[dvipsnames]{xcolor}

\usepackage{graphicx}
\usepackage{booktabs}
\usepackage{multirow}
\usepackage{subcaption}
\usepackage{pgfplots}
\pgfplotsset{compat=1.18}
\usepackage{tikz}
\usetikzlibrary{arrows.meta, decorations.pathreplacing, calc, fit}

\usepackage{enumitem}

\usepackage{algorithm}
\usepackage{algpseudocode}

\usepackage[colorlinks=true,linkcolor=blue!80!black,citecolor=blue!80!black,urlcolor=blue!80!black]{hyperref}

\newtheorem{theorem}{Theorem}
\newtheorem{proposition}{Proposition}
\newtheorem{lemma}{Lemma}
\newtheorem{definition}{Definition}
\newtheorem{remark}{Remark}
\newtheorem{corollary}{Corollary}

\newcommand{\R}{\mathbb{R}}

\newcommand{\zono}[1]{\langle #1 \rangle}

\newcommand{\operator}[1]{{\normalfont \texttt{#1}}}
\newcommand{\Oc}{\mathcal{O}}
\newcommand{\id}{\mathsf{id}}
\DeclareMathSymbol{\shortminus}{\mathbin}{AMSa}{"39}

\DeclareMathOperator{\softmax}{softmax}
\DeclareMathOperator{\diag}{diag}

\newcommand{\pz}{\mathcal{P}}
\newcommand{\cpz}{\mathcal{C}}

\title{Certified Mechanistic Interpretability:\\Lifting Single-Input Findings to Bounded Neighbourhoods}

\author{%
  Zhen Zhang\textsuperscript{*} \quad Yanliang Huang\textsuperscript{*} \quad
  Peng Xie \quad Wenyuan Wu \quad Amr Alanwar \\
  School of Computation, Information and Technology \\
  Technical University of Munich, Germany \\
  {\small\texttt{\{zhenzhang.zhang, yanliang.huang, p.xie, wenyuan.wu, alanwar\}@tum.de}} \\
  {\small\textsuperscript{*}Equal contribution.}
}
\date{}

\begin{document}

\maketitle

\begin{abstract}
Mechanistic interpretability reverse-engineers transformer circuits one input at a time, leaving observed mechanisms without guarantees over bounded input neighbourhoods. We address this gap with a framework based on constrained polynomial-zonotope (CPZ) propagation that lifts mechanistic-interpretability observations from a single input to certified statements over a bounded set of perturbations. Three internal-attention queries (top-$k$ stability, evidence mass, and attention entropy) are formulated as tractable programs over the simplex of attention weights, and CPZ propagation through transformer blocks is shown to preserve the softmax simplex and the LayerNorm zero-mean identity exactly. A recursive Jacobian zonotope construction extends the same certificates across layer depth by linearising the block stack at the input and avoids per-layer generator growth. We instantiate the framework on transformer attention; the resulting certificates offer a way to sharpen mechanistic statements that single-input inspection cannot resolve on its own, and to inform downstream decisions in regimes where empirical heuristics may be misleading.
\end{abstract}

\section{Introduction}
\label{sec:intro}

Mechanistic interpretability has identified specific transformer mechanisms (induction heads~\citep{olsson2022context}, duplicate-token heads, IOI circuits~\citep{wang2023interpretability}, factual-recall heads) by examining single inputs. But a single-input observation is just that: an observation. Whether the head's behaviour holds across nearby inputs has not been formalised, and asking the question can change published findings: heads that single-input attention inspection classifies identically can split sharply once a bounded perturbation neighbourhood is taken into account.

\begin{figure}[t]
\centering
\includegraphics[width=\linewidth]{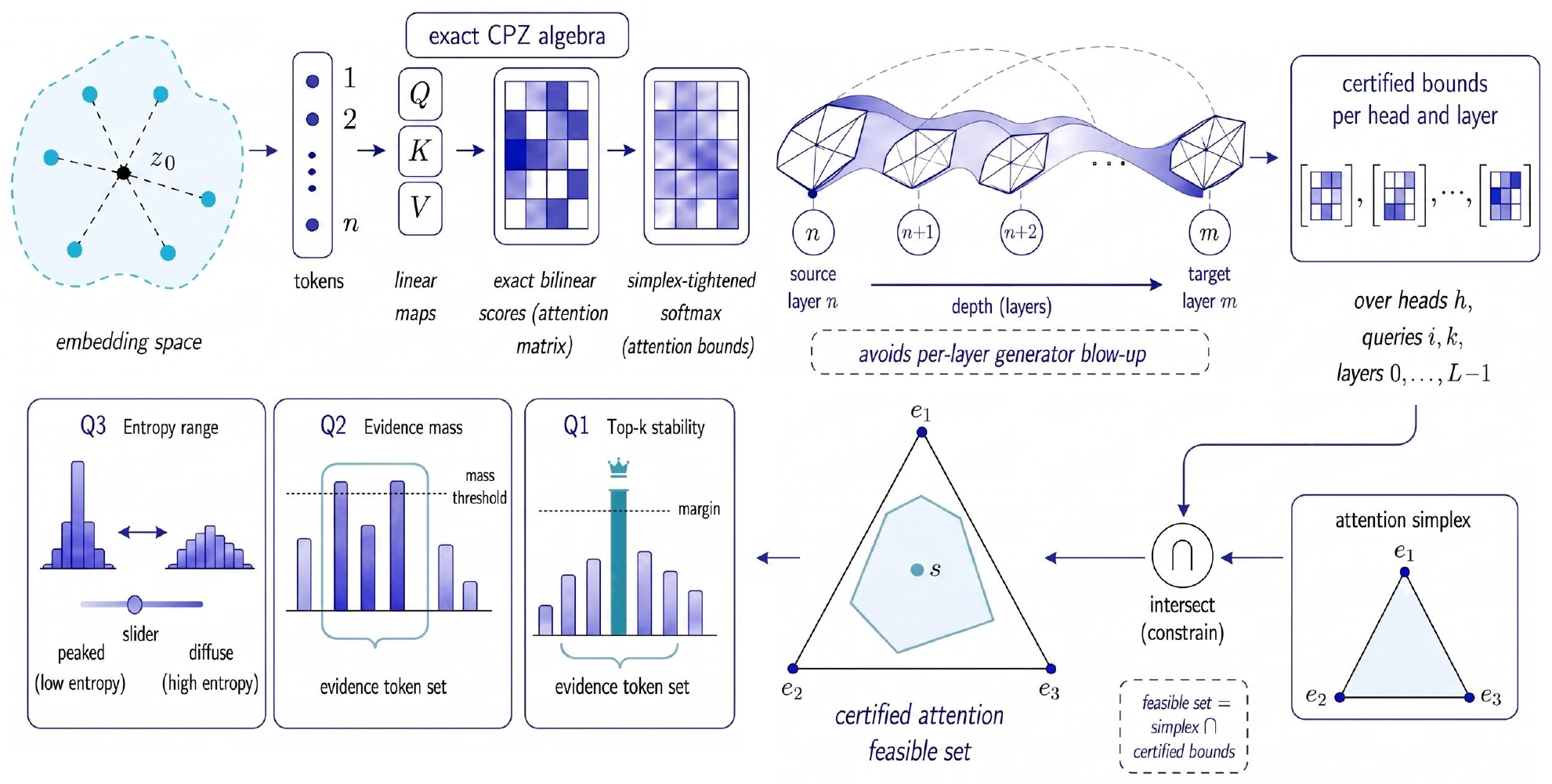}
\caption{\textbf{Pipeline overview.} A post-embedding perturbation set around the clean input $z_0$ is encoded as a CPZ and propagated through the transformer. \textbf{Top:} $Q^\top K$ stays exact under CPZ algebra; softmax is tightened by the simplex constraint $\sum_j s_{ij}{=}1$; a recursive Jacobian--zonotope connector (\S\ref{sec:jac_zono}) lifts bounds across layer depth. \textbf{Bottom:} the feasible attention polytope $R_\epsilon = \Delta_S \cap [\underline{s}, \overline{s}]$ supports three certified queries Q1/Q2/Q3 (\S\ref{sec:cert_queries}).}
\label{fig:overview}
\end{figure}

Consider a transformer with two attention heads at perfect accuracy on an induction task. Ablation rates them equally functional, and gradient norm declares one head much more sensitive than the other; standard gradient pruning would drop the smaller-gradient head. This decision misfires under bounded perturbation: the smaller-gradient head is the one whose attention actually shifts to track the duplicate token, and pruning it loses adversarial accuracy that pruning the other preserves. No empirical interpretability tool catches this in advance; a framework that bounds attention over the whole neighbourhood does.

Mechanistic interpretability and neural-network verification are adjacent but neither addresses internal mechanisms in a neighbourhood. Mechanistic interpretability~\citep{olsson2022context,wang2023interpretability} studies internal circuits by observation, without proving they hold in a neighbourhood. Neural-network verification~\citep{gowal2018effectiveness,zhang2018crown,singh2019abstract} issues certificates for output properties only; sequential polynomial-zonotope propagation~\citep{ladner2025towards} is intractable at BERT scale because generators grow with depth. Internal mechanisms have not been certified, and set-based methods have not reached pretrained transformer scale even on outputs.

We give a framework, \emph{certified mechanistic interpretability}, that lifts a single-input observation of an internal mechanism to a statement holding across a bounded perturbation neighbourhood at pretrained transformer scale (Fig.~\ref{fig:overview}). The certified mechanisms are attention-level: per-head top-$k$ stability, evidence-mass concentration, and attention-entropy bounds, formulated as tractable programs over the attention simplex. Constrained polynomial zonotopes (CPZs)~\citep{kochdumper2023constrained} preserve the simplex and LayerNorm identities through transformer operations, and a score-margin theorem reduces top-$1$ certification to a scalar sign check. A recursive Jacobian--zonotope connector linearises the block stack once at the input, avoids the per-layer generator blow-up that stops sequential propagation after one layer, and extends the certificates to pretrained-transformer scale.

\paragraph{Contributions.} We introduce \emph{certified mechanistic interpretability}, a framework that lifts single-input observations of transformer-internal mechanisms to statements holding over a bounded perturbation neighbourhood. We instantiate it through three attention-level queries---top-$k$ stability, evidence mass, and attention entropy---formulated as tractable programs over the attention simplex, and we give a constrained polynomial-zonotope propagation method that preserves the softmax simplex and the LayerNorm identity exactly, together with a recursive Jacobian-zonotope connector that extends the same certificates to pretrained-transformer depth. The resulting certificates separate published mechanistic findings that single-input inspection cannot distinguish, and inform pruning decisions where gradient-based heuristics fail.

\section{Background}
\label{sec:background}

\paragraph{Constrained Polynomial Zonotopes.}
\label{sec:cpz_def}
A CPZ~\citep{kochdumper2023constrained} represents a set as a polynomial expansion over factors $\alpha_k\in[-1,1]$ subject to polynomial constraints:
\begin{equation}
\label{eq:cpz}
\cpz = \bigg\{ c + \sum_{i=1}^{h} \Big(\textstyle\prod_{k=1}^{p} \alpha_k^{E_{k,i}}\Big) G_{\cdot,i} ~\bigg|~ \sum_{j=1}^{q} \Big(\textstyle\prod_{k=1}^{p} \alpha_k^{R_{k,j}}\Big) A_{\cdot,j} = b,\; \alpha_k \in [-1,1] \bigg\}.
\end{equation}
Center $c\in\R^n$, generator matrix $G$, exponent matrix $E$, constraint pair $(A,b)$, and constraint exponent matrix $R$ define the set; an identifier vector $\id$ labels factors so that \operator{mergeID}~\citep{kochdumper2023constrained} aligns factor spaces across operations. We additionally use independent generators $G_I$ with unconstrained factors $\beta_j\in[-1,1]$ to over-approximate non-polynomial operations. Tightening $G_I$ via constraints is central: when polynomial bounds $[\ell_{\text{dep}}, u_{\text{dep}}]$ intersect simplex/LayerNorm-induced bounds $[\ell_\cap, u_\cap]$, replacing the diagonal $G_I$ block by $\diag(\max(\ell_{\text{dep}}-\ell_\cap,\,u_\cap-u_{\text{dep}},\,0))$ (Eq.~\ref{eq:tighten_gi}) is sound and dominates the unconstrained bound at every coordinate.

\paragraph{CPZ algebra.} CPZs admit algebraically exact addition $\boxplus$ and element-wise multiplication $\odot$~\citep{zhang2026datadriven,zhang2025exactmult}: $\odot$ produces $h_1 h_2$ cross-term generators with exponents $E_{1}^{(\cdot,i)}+E_{2}^{(\cdot,j)}$, retaining polynomial dependency on the shared factor vector $\alpha$. The bilinear inner product $\cpz_1^\top \cpz_2 = \mathbf{1}^\top(\cpz_1 \odot \cpz_2)$ is thus exact, distinguishing CPZ from CROWN: the quadratic dependency $Q^\top K$ is preserved as a degree-$2$ polynomial in $\alpha$ where linear relaxation methods relax. Independent generators $G_I$ from softmax/LayerNorm are bounded conservatively; we tighten them at key points via structural constraints (\S\ref{sec:pz_attention}).

\paragraph{Notation.}
We consider a standard transformer encoder with $S$ tokens in $\R^d$ and $n_{\text{heads}}$ heads of dimension $d_h = d/n_{\text{heads}}$. Per head $h$: attention scores $a_{ij}^{(h)} = Q_i^{(h)\top} K_j^{(h)} / \sqrt{d_h}$, softmax weights $s_{ij}^{(h)} = \softmax_j(a_{i\cdot}^{(h)})$, and output $o_i^{(h)} = \sum_j s_{ij}^{(h)} V_j^{(h)}$. The block includes residual connections, LayerNorm, and a ReLU feedforward network.

\paragraph{Setup.} Synthetic experiments use a $2$-layer transformer encoder ($d{=}8$, $h{=}2$, $d_h{=}4$, $S{=}4$, $1{,}218$ parameters) trained to $100\%$ accuracy on a binary classification task; multi-layer scaling adds a $4$-layer synthetic transformer (App.~\ref{app:multilayer}); pretrained results use BERT-tiny on SST-$2$ and GPT-$2$ small ($124$M). All $S$ input tokens are perturbed jointly under $\ell_\infty$: $x \in [x_0 - \epsilon, x_0 + \epsilon] \subset \R^{S \times d}$. Embedding-space $\ell_\infty$ is the standard threat model for transformer verification~\citep{shi2020robustness, bonaert2021fast}; calibration (App.~\ref{app:emb_calibration}) places $\epsilon{=}0.001$--$0.002$ at $3$--$7\%$ of the nearest discrete token swap. We compare \textbf{CPZ} (ours), \textbf{PZ}~\citep{althoff2013reachability}, \textbf{CROWN}~\citep{xu2020automatic}, \textbf{IBP}, and \textbf{MC+PGD}: corner-biased Monte Carlo samples combined with multi-restart projected gradient descent~\citep{madry2018towards}, with per-table budgets given in each caption, used as a strong empirical upper bound on the certifiable rate rather than a ground truth.

\section{Certified Mechanistic Queries}
\label{sec:cert_queries}

Given attention weight bounds $s_j \in [\underline{s}_j, \overline{s}_j]$ from CPZ propagation, the simplex constraint $\sum_j s_j = 1$ with $s_j \geq 0$ is what lets us certify that a head's computational role is preserved under perturbation. The three queries below pose linear objectives (Q1 and Q2) or a concave one (Q3) on $\Delta_S \cap [\underline{s}, \overline{s}]$.

\subsection{Q1: Top-$k$ Attention Stability}
\label{sec:query_topk}

\begin{definition}[Certified Top-$k$ Attention Stability]
\label{def:topk_stability}
Given query token $i$, perturbation set $\mathcal{B}_\epsilon$, and softmax weights $s_{ij}(\delta)$, the top-$k$ pattern is certified stable if $\mathrm{top\text{-}k}(\{s_{ij}(\delta)\}_j) = T^*$ for all $\delta \in \mathcal{B}_\epsilon$, i.e., the $k$-token set receiving the highest weights is invariant.
\end{definition}

This certifies that a head's role (e.g.\ ``always attends most to the subject token'') is not an artefact of one input. A stronger variant certifies the ranking too:
\begin{corollary}[Certified Top-$k$ Ranking Preservation]
\label{cor:topk_ranking}
The ranking within $T^*$ is certified stable if for every pair $j, j' \in T^*$ with $s_{ij}^0 > s_{ij'}^0$ the simplex LP confirms $\max_{s \in \Delta_S \cap [\underline{s}, \overline{s}]} (s_{j'} - s_j) < 0$ ($\binom{k}{2}$ pairwise LPs in $\Oc(S)$ each; cross-head circuit AND of per-head certs in App.~\ref{app:extended_queries}).
\end{corollary}

\begin{proposition}[CPZ Certification via Simplex LP]
\label{prop:cert_attn}
Given softmax bounds $s_j \in [\underline{s}_j, \overline{s}_j]$ with the simplex constraint $\sum_j s_j = 1$, the top-$1$ attention target $j^*$ is certified stable if $\max_{k \neq j^*} \max_{s \in \Delta_S \cap [\underline{s}, \overline{s}]} (s_k - s_{j^*}) < 0$, where each inner LP is solvable in $\Oc(S)$ by greedy allocation. Without the simplex, the test degenerates to $\underline{s}_{j^*} > \max_{k \neq j^*} \overline{s}_k$, which is strictly weaker; see App.~\ref{app:query_proofs} for the proof.
\end{proposition}

\paragraph{Polynomial score margin.} Softmax monotonicity gives $\arg\max_j s_{ij}{=}\arg\max_j a_{ij}$, so top-$1$ reduces to a scalar sign check; CPZ exposes a closed-form, optimisation-free quadratic upper bound on the score margin:
\begin{lemma}[Quadratic CPZ upper bound]
\label{lem:quadratic_bound}
Let $p(\alpha) = c + \mathbf{a}^\top \alpha + \alpha^\top H \alpha$ for $\alpha \in [-1,1]^n$, with $H$ symmetric. Then $\overline{p} \coloneqq c + \|\mathbf{a}\|_1 + \sum_k \max(0, H_{kk}) + \sum_{k<\ell} 2|H_{k\ell}| \geq \max_{\alpha\in[-1,1]^n} p(\alpha)$, exploiting $\alpha_k^2\in[0,1]$ on the diagonal. The bound is sound (per-term suprema) but generally not tight; on Layer-$0$ score margins the gap to $200$-restart PGD is $\leq 0.5\%$ across Tab.~\ref{tab:attention_stability}. Proof: App.~\ref{app:query_proofs}.
\end{lemma}

\begin{theorem}[Sound Top-1 via Score-Margin CPZ]
\label{thm:margin_sound}
Let $\Delta_j(\alpha) = a_j(x_0 {+} \epsilon \alpha) - a_{j^*}(x_0 {+} \epsilon \alpha)$, with $\alpha \in [-1,1]^{Sd}$ identified with the CPZ factor vector of Eq.~\eqref{eq:cpz}. CPZ propagation yields $\Delta_j(\alpha) = c_\Delta + \mathbf{a}^\top \alpha + \alpha^\top H \alpha + \beta^\top G_I \mathbf{1}$, where $H$ is the degree-$2$ part from $Q^\top K$ and $G_I \geq 0$ collects independent generators from any LayerNorm preceding $Q^\top K$. If $G_I=0$ (post-LN, e.g., BERT-tiny Layer-$0$), Lemma~\ref{lem:quadratic_bound} gives a closed-form upper bound $\overline{\Delta}_j$; if $G_I\ne 0$ (pre-LN, e.g., GPT-$2$), adding $\|G_I\|_1$ remains sound (the $G_I$ contribution is bounded by Theorem~\ref{thm:ln}). In either case, $\overline{\Delta}_j < 0$ for every challenger $j \neq j^*$ certifies $j^*$ as top-$1$ over $\mathcal{B}_\epsilon(x_0)$.
\end{theorem}
\paragraph{Layer-$0$ matches MC.} On a $d{=}16$ model ($10$ samples, $160$ queries/$\epsilon$), CPZ-margin matches MC+PGD across all $6$ perturbation radii (Tab.~\ref{tab:attention_stability}); the advantage over auto-LiRPA CROWN ranges from $7.5$\,pp to ${>}40$\,pp (mean ${\sim}28$\,pp) since interval arithmetic destroys $Q^\top K$ cross-correlations. Proof of Theorem~\ref{thm:margin_sound}: App.~\ref{app:query_proofs}.

\subsection{Q2: Evidence Mass Certification}
\label{sec:query_mass}

\begin{proposition}[Certified Evidence Mass]
\label{prop:evidence_mass}
Given evidence tokens $\mathcal{E} \subseteq \{1, \ldots, S\}$ and threshold $\tau$, the attention mass is certified if $\min_{s \in \Delta_S \cap [\underline{s}, \overline{s}]} \sum_{j \in \mathcal{E}} s_j \geq \tau$. The simplex LP greedily allocates budget to non-evidence tokens first; the simplex-constrained minimum is strictly higher than $\sum_{j \in \mathcal{E}} \underline{s}_j$. The proof is in App.~\ref{app:query_proofs}.
\end{proposition}

\begin{corollary}[Head Specialization]
\label{prop:head_spec}
A natural variant certifies head specialization: given disjoint token groups $\mathcal{A}, \mathcal{B}$, a head is certified $\mathcal{A}$-specialized if $\min_{s \in \Delta_S \cap [\underline{s}, \overline{s}]} (\sum_{j \in \mathcal{A}} s_j - \sum_{j \in \mathcal{B}} s_j) > 0$. The minimum is computed by the same greedy LP as Prop.~\ref{prop:evidence_mass} with coefficients $+1$ on $\mathcal{A}$ and $-1$ on $\mathcal{B}$.
\end{corollary}

In experiments, $\tau$ expresses concentration above uniform $|\mathcal{E}|/S$: $|\mathcal{E}|{=}1$ uses $\tau{=}0.2$ ($1.6{\times}$ uniform), $|\mathcal{E}|{=}2$ uses $\tau{=}0.3$ ($1.2{\times}$). $\mathcal{E}$ is the clean-input top-attended position (Apps.~\ref{sec:bert},~\ref{sec:gpt2}) or fixed lookup $\{0,1\}$ (\S\ref{sec:exp_scaling}).

\subsection{Q3: Attention Entropy Certification}
\label{sec:query_entropy}

Q1--Q2 certify discrete structural properties; Q3 adds a continuous concentration measure via Shannon entropy $H(s) = -\sum_j s_j \log s_j$.
\begin{proposition}[Certified Attention Entropy Bounds]
\label{prop:entropy}
Given softmax bounds $s_j \in [\underline{s}_j, \overline{s}_j]$ with the simplex constraint, the certified entropy range $[\underline{H}, \overline{H}]$ is computed in $\Oc(S \log S)$: the lower bound enumerates vertices of the bounded simplex (concavity of $H$), and the upper bound solves a KKT bisection over the Lagrange multiplier $\lambda$ with $s_j^*(\lambda) = \min(\overline{s}_j, \max(\underline{s}_j, e^{\lambda - 1}))$. Without the simplex constraint, each $s_j$ varies independently in $[\underline{s}_j, \overline{s}_j]$, so the bounds collapse to the trivial range $[\underline{H}_{\mathrm{box}},\; \log S]$, where $\underline{H}_{\mathrm{box}} = 0$ whenever some $\underline{s}_j$ allows a degenerate point distribution; the upper bound $\log S$ is the uniform-distribution maximum and is reachable iff $1/S \in [\underline{s}_j, \overline{s}_j]$ for all $j$. Details are in App.~\ref{app:query_proofs}.
\end{proposition}

\paragraph{Capability gap.} Q1--Q2 are linear and Q3 concave on $\Delta_S \cap [\underline{s}, \overline{s}]$; without the simplex, all weights can saturate their upper bounds simultaneously, violating $\sum_j s_j{=}1$. This is qualitative, not merely tightness:
\begin{theorem}[Certification Separation]
\label{thm:separation}
For any $S \geq 2$ and simplex excess $E = \sum_j \overline{s}_j - 1 > 0$, there exist bounds $[\underline{s}, \overline{s}]$ where unconstrained certification fails but the simplex LP succeeds, with margin gap $\geq E/(S{-}1)$. \emph{Witness} ($S{=}3$, $E{=}0.6$): $s_1\!\in\![0.5,0.7]$, $s_2\!\in\![0.3,0.5]$, $s_3\!\in\![0.2,0.4]$ gives unconstrained margin $\overline{s}_2-\underline{s}_1{=}0$ (fails) vs.\ simplex-constrained $\max(s_2-s_1){=}-0.2$ (certifies). Proof: App.~\ref{app:query_proofs}.
\end{theorem}

\section{Scaling to Pretrained Transformers}
\label{sec:scaling}

\begin{figure}[t]
\centering
\includegraphics[width=\linewidth]{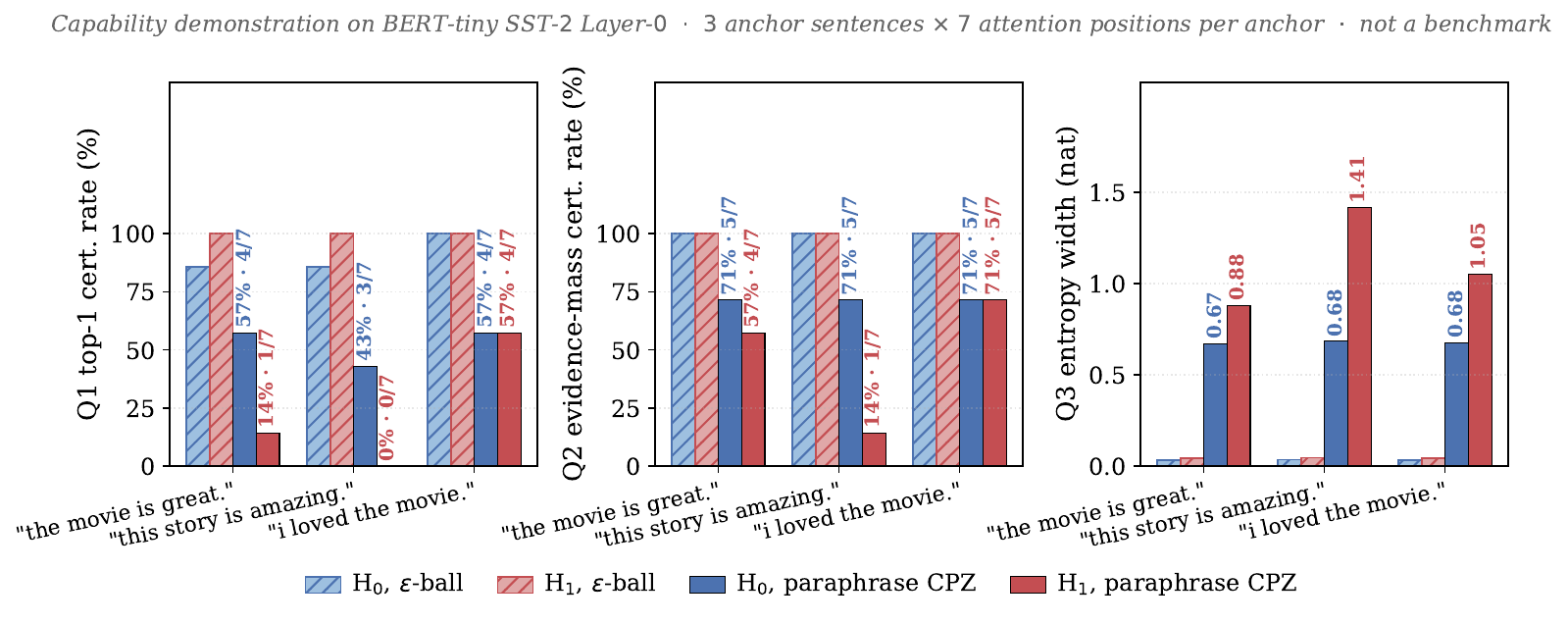}
\caption{\textbf{$\epsilon$-ball certification cannot distinguish BERT-tiny heads that a structured CPZ does.} BERT-tiny SST-$2$ Layer-$0$, three $5$-sentence paraphrase sets. Hatched bars ($\epsilon{=}0.002$ ball, App.~\ref{app:emb_calibration}): both heads saturated on Q1/Q2/Q3, indistinguishable. Solid bars (CPZ over the $5$-paraphrase polytope, construction in \S\ref{sec:exp_scaling}): \textbf{left}~Q1 top-$1$ stability, H$_0$ $43$--$57\%$ vs.\ H$_1$ $0$--$57\%$; \textbf{middle}~Q2 evidence mass ($\tau{=}0.2$), $71\%$ vs.\ $14$--$71\%$; \textbf{right}~Q3 entropy width, ${\sim}0.68$ vs.\ $0.88$--$1.41$ nat.}
\label{fig:paraphrase}
\end{figure}

Lifting the certificates of \S\ref{sec:cert_queries} through depth to BERT-tiny / GPT-$2$ scale needs: propagation rules exploiting transformer structural identities (\S\ref{sec:pz_attention}), a layer connector without per-layer generator blow-up (\S\ref{sec:jac_zono}), a sound verification pipeline (\S\ref{sec:machinery}), and empirical validation (\S\ref{sec:exp_scaling}). Fig.~\ref{fig:overview} shows the flow.

\subsection{CPZ Propagation Through the Transformer Block}
\label{sec:pz_attention}

We build one CPZ per token and propagate through the block. Affine maps and bilinear $Q^\top K$ are exact under the CPZ algebra of \S\ref{sec:background} when no LayerNorm precedes $Q^\top K$ (post-LN, e.g., BERT-tiny); pre-LN architectures (GPT-$2$) add a $G_I$ term per Theorem~\ref{thm:margin_sound}. Softmax, LayerNorm output, and ReLU/GELU introduce relaxation. The framework accepts any CPZ-representable perturbation set, unlike axis-aligned IBP/CROWN.

\paragraph{Structural-identity preservation.} For top-$1$, certification reduces to $\overline{\Delta}_j<0$ on the score-margin polynomial via Theorem~\ref{thm:margin_sound}, bypassing softmax; for Q2/Q3 we encode the simplex as a linear CPZ constraint with the $\Oc(S\log S)$ LP of Prop.~\ref{prop:lp_simplex}, persisting through all CPZ operations after \operator{mergeID} alignment (Theorem~\ref{thm:constraint_preservation}). LayerNorm contributes the algebraic identity $\sum_k z_k^2 = d$ ($z=(x-\mu)/\sigma$, $\epsilon_{\mathrm{LN}}{=}10^{-5}$); Theorem~\ref{thm:ln} gives $|y_k-\beta_k|\le|\gamma_k|\sqrt{d-1}$, and Theorem~\ref{thm:width_stability} bounds per-layer width by $2|\gamma_k^{(\ell)}|\sqrt{d-1}$ independent of $\epsilon$, depth, and accumulated generators (App.~\ref{app:sec4_proofs}).

\paragraph{Other components.} Attention output uses center--residual at Layer-$0$ or Taylor CPZ at depth (App.~\ref{app:attn_out_taylor}); FFN activations (ReLU/GELU) use a per-neuron hybrid quadratic/DeepPoly relaxation. CPZ soundness is MC-independent; MC+PGD budgets in App.~\ref{app:details}.

\subsection{Recursive Jacobian-Zonotope Connector}
\label{sec:jac_zono}

For Layer~$\ell{>}0$ we need a CPZ enclosure of $f=\text{Layer}_0\circ\dots\circ\text{Layer}_{\ell-1}$. Diagonal-generator CPZs destroy cross-dimensional correlations; we instead build the Jacobian zonotope from $\epsilon J(x_0)$:
\begin{equation}
\label{eq:jac_zono}
\cpz_t = \bigg\{ f_t(x_0) + \sum_{j \in \mathcal{K}} \alpha_j \,\epsilon J_t^{(\cdot,j)} + \alpha_{\mathrm{rem},t}\, r_t \,\Big|\, \alpha_j, \alpha_{\mathrm{rem},t} \in [-1,1] \bigg\},
\end{equation}
where $\mathcal{K}$ retains the top-$K_J$ columns by $\ell_1$ norm and $r_t$ absorbs dropped columns plus linearisation error (App.~\ref{app:machinery}); $K_J$ denotes the Jacobian-zonotope rank, distinct from the order-reduction parameter $k$ in \S\ref{sec:machinery}. Three design choices control tightness: shared factor IDs across tokens enable cross-token cancellation in $Q^\top K$; the remainder is dependent (cancels in differences); one remainder per token caps generators at $K_J{+}1$. $K_J{=}64$ (BERT-tiny) and $K_J{=}128$ (GPT-2) yield near-tight certificates; a $K_J\in\{32,\dots,1024\}$ sweep on BERT-tiny output verification finds the certified rate stable at $90$--$96.7\%$ (App.~\ref{app:scalability_bert}). Jacobian-zonotope linearisation is standard in reachability~\citep{althoff2013reachability}; the contribution is the transformer-adapted shared-ID + per-token-remainder construction combined with CPZ constraint injection.

\paragraph{Sound-with-fallback.} Every operation in \S\ref{sec:pz_attention} gives closed-form sound bounds; the Jacobian connector adds Prop.~\ref{prop:remainder_analytical}'s closed-form spectral Lipschitz fallback. Production uses a tighter sampled estimate with $1.5{\times}/2.0{\times}$ safety factors (BERT-tiny / GPT-$2$ end-to-end), calibrated against MC+PGD boundary cases (App.~\ref{app:output_verification}); the sampled estimate is $\leq 1.33\%$ of the analytical envelope on the $4$-layer synthetic transformer (Tab.~\ref{tab:analytical_remainder}).

\subsection{Verification Pipeline and Soundness Boundary}
\label{sec:machinery}

Exact CPZ multiplication produces $h_1 h_2$ cross-terms, prohibitive in deep pipelines. We apply the standard CORA order reduction~\citep{kochdumper2023constrained}: sort cross-terms by $\|g_\ell\|_1$, retain top-$k$ as dependent (preserving polynomial structure), Girard-reduce~\citep{girard2005reachability} the remainder into $G_{I,\text{trunc}}=\diag(\sum_{\ell>k}|g_\ell|)$. $k{=}h_1 h_2$ recovers exactness, $k{=}0$ recovers IBP; combined with constraint tightening, $k{=}16$ suffices for near-exact precision. Algorithm~\ref{alg:verify} gives the pipeline; complexity is $\Oc(L(S^2 d r^2 + S d d_{\mathrm{ff}} r))$ vs.\ $\Oc(L S^2 d^2)$ for IBP (Prop.~\ref{prop:complexity}, App.~\ref{app:machinery}).

\paragraph{Soundness boundary.} All certified results are sound; the boundary marks tightness. Layer-$0$ Q1 (Theorem~\ref{thm:margin_sound}+Props.~\ref{prop:evidence_mass},\ref{prop:entropy} for post-LN; with $G_I$ from Theorem~\ref{thm:ln} for pre-LN GPT-$2$) and pruning Step~2 are fully analytical; multi-layer and end-to-end rely on the sampled remainder above. No MC+PGD-found counterexample across $450$ (sample, $\epsilon$) pairs ($100$ BERT-tiny + $50$ GPT-$2$ at $3$ radii).

\subsection{BERT-tiny and GPT-2: Cross-Paper Replication and Scaling}
\label{sec:exp_scaling}

We probe GPT-2-small Layer~$0$ to (i) test \citet{wang2023interpretability}'s Duplicate-Token-Head claim, and (ii) show the three queries measure independent aspects of attention robustness. Layer-$0$ Q1 uses Theorem~\ref{thm:margin_sound} (bilinear for post-LN BERT-tiny; with $G_I$ for pre-LN GPT-$2$) with no Jacobian remainder; Q2/Q3 use Props.~\ref{prop:evidence_mass}--\ref{prop:entropy}.

\paragraph{Cross-paper validation (60 IOI prompts).} \citet{wang2023interpretability} label GPT-$2$-small Layer-$0$ heads $0.1$ and $0.10$ as Duplicate-Token Heads. We construct $60$ IOI-pattern sentences with a duplicated proper noun and certify Q1 at the second-duplicate position. Clean attention confirms Wang et al.\ ($0.10$: $60/60$, $0.1$: $36/60$); CPZ adds a robustness split (Tab.~\ref{tab:wang_probe}, left): at $\epsilon{=}5{\times}10^{-4}$ head $0.10$ retains $100\%$ certified Q1 while $0.1$ collapses to $1.7\%$; at $\epsilon{=}10^{-3}$, $0.10$ holds $95\%$; at $\epsilon{=}2{\times}10^{-3}$ all four candidates fail (probe not vacuously passing). Wang's clean classification is correct, but the heads differ by a $98$\,pp robustness gap invisible to single-input inspection.

\paragraph{Multi-query orthogonality (30 induction prompts).} On $30$ sequences with last-token=first, $10$ of $12$ Layer-$0$ heads pass all three queries; the remaining two show orthogonal failure modes (Tab.~\ref{tab:induction_probe}): $\mathbf{H_5}$ passes Q2/Q3 but fails Q1 (drifting target); $\mathbf{H_{10}}$ passes Q1/Q2 but fails Q3 (drifting concentration). Both heads pass single-input checks; only the three queries together separate them.

\paragraph{Cross-head circuit certification (60 IOI prompts).} The per-head AND construction of Remark~\ref{rem:crosshead} lifts the probe to circuit level on the same $60$ IOI prompts (Tab.~\ref{tab:wang_probe}, right). Wang et~al.'s published DTH pair $\{0.1,0.10\}$ certifies $1/60$ at $\epsilon{=}5{\times}10^{-4}$ and $0/60$ at larger $\epsilon$, with head $0.1$ acting as the joint-robustness bottleneck (it still passes clean-input attention; left). Empirical-stable pairs do certify jointly: $\{0.3,0.4\}$ holds $60/60$ at $\epsilon{\leq}10^{-3}$ and $54/60$ at $\epsilon{=}2{\times}10^{-3}$; the triple $\{0.3,0.4,0.9\}$ holds $60/60$, $59/60$, $45/60$. Joint $\epsilon$-robustness is therefore strictly stronger than clean-input agreement and adds a robustness lens on top of (not in conflict with) discovery-time circuit identification.

\begin{table}[h]
\centering
\caption{Per-head (left) and cross-head circuit (right) Q1 cert rates on the same $60$ IOI prompts. $^*$ marks Wang-classified DTHs. The Wang pair satisfies clean attention but not joint $\epsilon$-robustness; empirically-stable pairs/triples certify jointly at near-saturation.}
\label{tab:wang_probe}
\label{tab:cross_head_circuit}
\footnotesize
\setlength{\tabcolsep}{3pt}
\renewcommand{\arraystretch}{0.92}
\begin{minipage}[t]{0.52\linewidth}\centering
\begin{tabular}{lcccc}
\toprule
Head & Clean & $5{\times}10^{-4}$ & $10^{-3}$ & $2{\times}10^{-3}$ \\
\midrule
$0.1^*$  & $36/60$ & $\mathbf{1.7}$ & $\mathbf{0}$  & $0$ \\
$0.5$    & $59/60$ & $58$           & $3$           & $0$ \\
$0.6$    & $31/60$ & $47$           & $7$           & $0$ \\
$0.10^*$ & $60/60$ & $\mathbf{100}$ & $\mathbf{95}$ & $0$ \\
\bottomrule
\end{tabular}
\end{minipage}\hfill
\begin{minipage}[t]{0.46\linewidth}\centering
\begin{tabular}{lccc}
\toprule
Circuit & $5{\times}10^{-4}$ & $10^{-3}$ & $2{\times}10^{-3}$ \\
\midrule
$\{0.1,0.10\}$    & $1/60$  & $0/60$  & $0/60$ \\
$\{0.3,0.4\}$     & $\mathbf{60/60}$ & $\mathbf{60/60}$ & $\mathbf{54/60}$ \\
$\{0.3,0.4,0.9\}$ & $\mathbf{60/60}$ & $\mathbf{59/60}$ & $\mathbf{45/60}$ \\
\bottomrule
\end{tabular}
\end{minipage}
\end{table}

\paragraph{Method comparison across scales.} Tab.~\ref{tab:main_results} summarises headline cert rates: CPZ is the only method whose Q1 rate stays above $40\%$ on GPT-$2$ Layer-$1$. End-to-end verification at $124$M is, to our knowledge, not previously reported. Scaling to GPT-$2$ medium ($355$M, $d{=}1024$) keeps CPZ sound at $41.1\%$; the gap to MC's $90.9\%$ and the ${\sim}2700$\,s/sample reflect a wider-generator cost ($1.6\%$ of $8{,}192$ Jacobian columns retained at $K_J{=}128$, vs.\ $2.1\%$ at $d{=}768$), not a failure mode (Tab.~\ref{tab:gpt2_medium}). The Jacobian--zonotope connector (\S\ref{sec:jac_zono}) is why CPZ continues to track MC where interval arithmetic loses $Q$-$K$ correlations.

\begin{table}[t]
\centering
\caption{\textbf{Main results.} Top-$1$ attention/output stability (\%) across model sizes; ``--'' = baseline does not apply at this scale. MC = MC+PGD, empirical upper bound. CROWN $<$ IBP at GPT-$2$ Layer-$1$ as input-dependent argmax explodes CROWN's softmax back-substitution at $d{=}768$ (App.~\ref{app:alphacrown_comparison}). $95\%$ Clopper--Pearson CIs in App.~\ref{app:details}.}
\label{tab:main_results}
\footnotesize
\setlength{\tabcolsep}{4pt}
\renewcommand{\arraystretch}{0.9}
\begin{tabular}{l c r r r r r}
\toprule
Setting & $\epsilon$ & \textbf{CPZ} & MC & CROWN & IBP & sec/sample \\
\midrule
\multicolumn{7}{l}{Internal Layer queries (Q1 top-$1$ attention stability)} \\
Synthetic $d{=}8$, Layer-$0$  & $2{\times}10^{-2}$ & $\mathbf{85.0}$ & $85.0$ & $72.5$ & $57.5$ & $<\!0.1$ \\
BERT-tiny ($d{=}128$), Layer-$1$ & $2{\times}10^{-3}$ & $\mathbf{92.5}$ & $95.0$ & $51.2$ & $\phantom{0}3.8$ & $0.7$ \\
GPT-$2$ ($d{=}768$), Layer-$0$  & $1{\times}10^{-3}$ & $\mathbf{70.2}$ & $72.7$ & $43.5$ & $52.3$ & $\sim\!5$ \\
GPT-$2$ ($d{=}768$), Layer-$1$  & $1{\times}10^{-3}$ & $\mathbf{47.9}$ & $78.1$ & $\phantom{0}0.2$ & $26.9$ & $\sim\!22$ \\
GPT-$2$ medium ($d{=}1024$), Layer-$1$ & $5{\times}10^{-4}$ & $\mathbf{41.1}$ & $90.9$ & $\phantom{0}0.8$ & $10.7$ & $\sim\!2700$ \\
\midrule
\multicolumn{7}{l}{End-to-end output verification (full classification / next-token)} \\
BERT-tiny SST-$2$ output       & $2{\times}10^{-3}$ & $\mathbf{95.0}$ & $98.0$ & $80.0^{\dagger}$ & -- & $0.7$ \\
GPT-$2$ next-token, top-$K{=}20$ direct$^{\ddagger}$  & $2{\times}10^{-3}$ & $\mathbf{90.0}$ & $90.0$ & --     & --     & $\sim\!18$ \\
\bottomrule
\end{tabular}\\[2pt]
{\scriptsize $^{\dagger}$ $\alpha$-CROWN on quick-GELU surrogate (not like-for-like; App.~\ref{app:alphacrown_comparison}). $^{\ddagger}$ Direct CPZ argmax over top-$K{=}20$ at $x_0$ ($45/50$); a $2.0{\times}$-safety Lipschitz tail on the remaining ${\sim}50{,}237$ entries tightens to $20/50$ full-vocabulary ($40\%$, tail-bound dominated; App.~\ref{app:output_verification}).}
\end{table}

\paragraph{Pretrained transformers.} On a $4$-layer synthetic, recursive CPZ matches MC within $0$--$7.5$\,pp at every layer while sequential CPZ collapses to $0\%$ at Layer-$2/3$ (App.~\ref{app:multilayer}); CROWN-family methods cannot back-substitute through softmax with input-dependent argmax (App.~\ref{app:alphacrown_comparison}). No unsound certification across $450$ (sample, $\epsilon$) pairs ($100$ BERT-tiny $+\ 50$ GPT-$2$ at $3$ radii, calibrated against MC+PGD boundary cases).

\paragraph{Layer-$1$ gap and closing paths.} The $30.2$\,pp CPZ--MC gap at GPT-$2$ Layer-$1$ comes from the $K_J{=}128$ Jacobian-column truncation ($1/48$ of $S\cdot d{=}6{,}144$ at $d{=}768$); on BERT-tiny ($K_J{=}64$, $1/16$) the gap is small and the $K_J$-sweep is stable at $90$--$96.7\%$ (App.~\ref{app:scalability_bert}). Two engineering paths close the gap without new theory: (i)~tighter Lipschitz constants via interval Hessian~\citep{zhang2019recurjac}; (ii)~more retained Jacobian columns via amortised cross-term computation.

\paragraph{$\alpha$-CROWN and connector ablation.} $\alpha$-CROWN~\citep{xu2020automatic} cannot natively express Q1/Q2/Q3 (no simplex constraint); on a quick-GELU surrogate of BERT-tiny it certifies $80$--$90\%$ in $11$--$12$\,s vs.\ CPZ's $95$--$98\%$ in $0.7$\,s on the original (not like-for-like; App.~\ref{app:alphacrown_comparison}). The recursive Jacobian-zonotope connector is why CPZ scales: sequential CPZ collapses to $\leq 12.5\%$ on the $4$-layer synthetic and $0\%$ on GPT-$2$ Layer-$1$, while the connector recovers all cells within $30$\,pp of MC+PGD (Tab.~\ref{tab:connector_ablation}, App.~\ref{app:multilayer}). CPZ-guided pruning attains $84.5\%$ vs.\ $78.5\%$ for gradient-norm across $20$ seeds (Wilcoxon $p{=}0.0042$, \S\ref{sec:case_study}).

\paragraph{Beyond $\epsilon$-balls.} The framework accepts arbitrary CPZ-representable sets: five paraphrases (common length $S$) give $Z_t=\overline{x}_t+\sum_{i=1}^{5}\alpha_i(x^{(i)}_t-\overline{x}_t)$, $\alpha_i\in[-1,1]$, enclosing the $5$-vertex convex hull with $\ell_\infty$ spread up to $3{,}500{\times}$ the calibrated ball (Fig.~\ref{fig:paraphrase}). The $\epsilon$-ball saturates all queries on both BERT-tiny SST-$2$ heads; the paraphrase CPZ separates them on Q1/Q2/Q3---invisible to axis-aligned methods.

\section{Pruning Case Study: A Minimal Proof-of-Concept}
\label{sec:case_study}

With internal queries certifiable, we ask what they enable downstream. This is a minimal proof-of-concept exposing a structural failure mode of gradient-norm pruning, not a general recipe (a $1{,}218$-parameter synthetic induction model with sharp head specialisation; transfer to fine-tuned LLM heads is open, Limitations). The model is a $2$-layer, $2$-head transformer ($d{=}8$, $S{=}6$) trained to perfect accuracy; both heads hit $100\%$ clean so standard metrics treat them as interchangeable.

\paragraph{Step 1: Empirical baselines fail.} Per-head ablation drops accuracy by $0$ for either head. Input gradient norms suggest Head~1 is ${\sim}5{\times}10^4{\times}$ more sensitive ($0.131$ vs.\ $2.7{\times}10^{-6}$, well above FP32 noise floor and reproducible across seeds; App.~\ref{app:adv_validation}); a gradient criterion would prune Head~0. Clean attention entropies are similar ($1.69$ vs.\ $1.58$). None of these signals predicts which head breaks under perturbation (Tab.~\ref{tab:empirical_baselines}).

\paragraph{Step 2: Certified queries expose hidden head specialisation.} Per-head CPZ (Tab.~\ref{tab:per_head}) reveals two structurally different attention regimes in the same perturbation ball: H1's attention pattern is highly stable (Q1 certifies $97\%$ vs.\ $72\%$ for H0 at $\epsilon{=}0.02$, Q2 $58\%$ vs.\ $25\%$, Q3 entropy width $0.042$ vs.\ $0.053$, $25{\times}$ tighter than IBP's vacuous $\approx 1.46$), whereas H0's attention is \emph{responsive}---it shifts as the input varies inside the ball. Gradient norm and CPZ certify orthogonal properties: gradient measures local loss sensitivity (and flags H1 as important), CPZ measures worst-case attention dynamics (and isolates H0 as the head whose pattern actually tracks input content). The induction task needs content-responsive attention to localise the duplicate, so the next step tests whether keeping the responsive H0 or the stable H1 yields the more adversarially-robust $1$-head model.

\begin{table}[h]
\centering
\caption{CPZ per-head certified queries on the induction model (Step~2 of \S\ref{sec:case_study}; $\tau{=}0.3$ for Q2 here vs.\ $\tau{=}0.2$ elsewhere). H1's attention is consistently \textbf{stable} (bold), H0's is responsive within the ball; high stability here signals an inert head, since induction needs content-tracking attention. Q1/Q2 in \%; Q3 in nat. IBP rows omitted: Q1/Q2 within $10$\,pp of CPZ, Q3 entropy bounds vacuous at $1.45$--$1.47$ vs.\ CPZ's $0.02$--$0.14$ ($10$--$70{\times}$ tighter).}
\label{tab:per_head}
\footnotesize
\setlength{\tabcolsep}{4pt}
\renewcommand{\arraystretch}{0.9}
\begin{tabular}{l ccc ccc ccc}
\toprule
 & \multicolumn{3}{c}{$\epsilon{=}0.01$} & \multicolumn{3}{c}{$\epsilon{=}0.02$} & \multicolumn{3}{c}{$\epsilon{=}0.05$} \\
\cmidrule(lr){2-4} \cmidrule(lr){5-7} \cmidrule(lr){8-10}
Head & Q1 & Q2 & Q3 & Q1 & Q2 & Q3 & Q1 & Q2 & Q3 \\
\midrule
H0 & $88$            & $27$            & $.021$ & $72$            & $25$            & $.042$ & $55$            & $18$            & $.105$ \\
H1 & $\mathbf{97}$   & $\mathbf{68}$   & $.027$ & $\mathbf{97}$   & $\mathbf{58}$   & $.053$ & $\mathbf{88}$   & $\mathbf{48}$   & $.134$ \\
\bottomrule
\end{tabular}
\end{table}

\paragraph{Step 3: Adversarial validation and pruning.} Targeted PGD ($20{\times}50$) confirms the structural distinction: H0 flips $3.7{\times}$ more attention positions than H1 at $\epsilon{=}0.05$ ($18.3\%$ vs.\ $5.0\%$), reflecting its content-responsive role. Pruning the inert H1 and keeping the responsive H0 (CPZ-guided) attains $\mathbf{100\%}$ adversarial accuracy at $\epsilon{\in}\{0.01,0.02,0.05\}$; pruning H0 and keeping H1 instead (gradient-guided, since H0 has near-zero gradient) drops to $100/90/90\%$ as the model loses content-responsive attention (Tab.~\ref{tab:certified_pruning}). Multi-seed (Tab.~\ref{tab:multiseed}) confirms $84.5\%$ vs.\ $78.5\%$ over $20$ seeds, Wilcoxon $p{=}0.0042$.

\begin{table}[h]
\centering
\caption{Per-seed adversarial accuracy (\%) under CPZ-guided vs.\ gradient-guided pruning across $20$ seeds, restricted to disagreement seeds at $\epsilon{=}0.05$ ($8/20$); rightmost column reports the all-seed mean. CPZ-guided pruning never underperforms gradient-guided (Wilcoxon signed-rank one-sided, $p{=}0.0042$).}
\label{tab:multiseed}
\footnotesize
\setlength{\tabcolsep}{3pt}
\renewcommand{\arraystretch}{0.9}
\begin{tabular}{l rrrrrrrr | r}
\toprule
Seed & 42 & 2024 & 31415 & 271828 & 11 & 23 & 67 & 101 & All-seed mean \\
\midrule
CPZ-guided      & 100 & 80 & 100 & 90 & 100 & 100 & 100 & 100 & $\mathbf{84.5}$ \\
Gradient-guided &  90 & 50 &  90 & 60 &  90 &  90 &  90 &  90 & $78.5$ \\
$\Delta$        & $+10$ & $+30$ & $+10$ & $+30$ & $+10$ & $+10$ & $+10$ & $+10$ & $+6.0$ \\
\bottomrule
\end{tabular}
\end{table}

\section{Related Work}
\label{sec:related}

\paragraph{Mechanistic interpretability.} Prior work identifies internal structures (induction heads~\citep{olsson2022context}, IOI circuits~\citep{wang2023interpretability}, factual recall) by single-input empirical analysis, without proving persistence under perturbation. The closest formal attempt~\citep{gross2024compact} encodes hand-crafted invariants for a small Max-of-$K$ transformer; we instead certify the mechanisms themselves with an architecture-agnostic framework that changes downstream decisions (\S\ref{sec:case_study}).

\paragraph{Neural-network verification.} Complete methods~\citep{bunel2020branch,katz2017reluplex,katz2019marabou,wang2021beta} scale poorly to attention layers; scalable incomplete methods~\citep{gowal2018effectiveness,zhang2018crown,singh2019abstract} certify outputs only. Transformer verifiers~\citep{shi2020robustness,bonaert2021fast} relax $Q^\top K$ via linear bounds, destroying the structure CPZ preserves; Lipschitz approaches~\citep{kim2021lipschitz} yield scalar envelopes only.

\paragraph{Set-based methods.} Zonotopes~\citep{singh2018fast}, star sets~\citep{tran2020verification}, polynomial zonotopes~\citep{althoff2013reachability,kochdumper2023open,ladner2024exponent}; \citet{ladner2025towards} extends PZ to transformers but reports BERT-scale intractable for sequential unconstrained PZ. Our recursive Jacobian-zonotope (\S\ref{sec:jac_zono}) reaches this scale with one linearisation; CPZ constraint injections~\citep{kochdumper2023constrained,zhang2026datadriven,zhang2025exactmult} make the simplex queries tractable; only CPZ satisfies all five capability dimensions (Tab.~\ref{tab:capability_matrix}, App.~\ref{app:capability_matrix}: internal queries, simplex constraint, exact $Q^\top K$, LLM-internal scale, downstream guidance).

\paragraph{Conclusion.}
\label{sec:discussion}
\label{sec:conclusion}
Once internal attention mechanisms become certifiable, two questions become tractable: (1) when can an interpretability finding be trusted as a property of the model rather than of the producing input, and (2) which downstream decisions should follow certified mechanism behaviour rather than empirical heuristics. This paper provides a first answer on attention-level mechanisms: no MC+PGD counterexample across $450$ (sample, $\epsilon$) pairs on BERT-tiny and GPT-$2$ ($124$M) under a $10{\times}50$ adversarial budget and a $2.0{\times}$ sampled-remainder safety factor calibrated against boundary cases (App.~\ref{app:output_verification}). Circuit-level certification across composing heads is the natural next step.

\paragraph{Limitations.}
\label{sec:limitations}
(i)~the CPZ--MC tightness gap grows with $d$ ($30.2$\,pp at GPT-$2$ Layer-$1$, dominated by Jacobian-zonotope truncation at $K_J{=}128$ of $6{,}144$); (ii)~the $40\%$ end-to-end full-vocab rate is dominated by the conservative Lipschitz tail bound, not the core CPZ machinery (direct $90\%$ on top-$K{=}20$); (iii)~multi-layer recursive certification demonstrated up to $L{=}4$ on a synthetic transformer; (iv)~the pruning case study is a minimal proof-of-concept on a $1{,}218$-parameter $2$-head synthetic model, not a recipe for fine-tuned LLMs (whether the gradient-vs-CPZ advantage transfers to real-model attention heads is open); (v)~discrete-token and counterfactual analyses are out of scope; (vi)~scaling experiments use a single training seed for the trained transformer; per-query Clopper--Pearson binomial CIs are reported in App.~\ref{app:details}, but cross-seed variance is unmeasured.

\bibliographystyle{plainnat}
\bibliography{ref_MT}

\appendix

\section{Proofs and Theoretical Results}
\label{app:proofs_and_theoretical_results}

\subsection{Proofs and Deferred Results for \S\ref{sec:pz_attention}}
\label{app:sec4_proofs}

This appendix collects the deferred theorem statements and their proofs from \S\ref{sec:pz_attention}. The first subsection restates the constraint preservation, width stability, and tightness results referenced in the main text; the remaining subsections supply their proofs together with the supporting lemmas on the convex combination bound, simplex exactness, the quantitative tightness gap, the CPZ--linear-relaxation comparison for bilinear scores, and box tightening.

\subsubsection{Theorem Statements Deferred from \S\ref{sec:pz_attention}}

\begin{theorem}[Constraint Preservation and Accumulation]
\label{thm:constraint_preservation}
Affine maps, addition, and multiplication preserve CPZ constraints ($A\alpha^R = b$) after \operator{mergeID} alignment. Constraints injected at layer $\ell$ remain valid through all later layers and accumulate monotonically.
\end{theorem}

\begin{proposition}[LP Simplex Tightening]
\label{prop:lp_simplex}
For $o_{i,k} = \sum_j s_{ij} V_{j,k}$ with $\sum_j s_{ij}=1$, $s_{ij}\in[\underline{s}_{ij},\overline{s}_{ij}]$, the bound
\begin{equation}\label{eq:lp_bound}
\underline{o}_{i,k}^{\mathrm{LP}} = \min_{s \in \Delta_S \cap [\underline{s}, \overline{s}]}\sum_j s_j V_{j,k}^{\mathrm{lb}}
\end{equation}
is solvable in $\Oc(S\log S)$ and gives the \emph{exact} hull over the constrained simplex. We tighten $G_I$ via
\begin{equation}\label{eq:tighten_gi}
G_{I,\text{new}} = \diag\big(\max(\ell_{\text{dep}} - \ell_\cap,\; u_\cap - u_{\text{dep}},\; 0)\big).
\end{equation}
\end{proposition}

\begin{theorem}[LayerNorm Dimension Bound]
\label{thm:ln}
For $z = (x-\mu)/\sigma$ with $\sum_k z_k = 0$, $\sum_k z_k^2 = d$: $|z_k| \leq \sqrt{d-1}$, and after affine $y_k = \gamma_k z_k + \beta_k$, $|y_k - \beta_k| \leq |\gamma_k|\sqrt{d-1}$.
\end{theorem}

\begin{theorem}[Width Stability]
\label{thm:width_stability}
Assume each transformer block ends with a LayerNorm whose output is the measurement point. Under CPZ propagation with LayerNorm tightening, the post-LN width at the end of block $\ell$ satisfies $w_k^{(\mathrm{CPZ},\ell)} \leq 2|\gamma_k^{(\ell)}|\sqrt{d-1}$, independent of $\epsilon$, depth, and accumulated generators. The bound applies only at LN-output positions; intermediate widths inside attention or FFN sub-blocks are not bounded by this theorem. Unconstrained propagation grows with prior-layer over-approximation.
\end{theorem}

\subsubsection{Proof of Constraint Preservation (Theorem~\ref{thm:constraint_preservation})}
All three operations (affine, addition, multiplication) transform generator matrices $(c, G, G_I)$ but not the factor vector $\alpha$. Since constraints $A\alpha^R = b$ restrict only $\alpha$ (and identifiers are preserved by \operator{mergeID}), they remain valid in the image. For multiplication, cross-term generators $G_{\cdot,i}\odot G_{\cdot,j}$ with exponents $E_{\cdot,i}+E_{\cdot,j}$ are polynomial in the same $\alpha$, so constraints continue to restrict the feasible factor space.

\subsubsection{Convex Combination Bound}
\begin{theorem}[Convex Combination Bound]
\label{thm:convex}
Let $s_{ij} = \softmax_j(a_{i\cdot})$. Then $o_i = \sum_j s_{ij} V_j$ satisfies, for each dimension $k$, $\min_j V_{j,k}^{\mathrm{lb}} \leq o_{i,k} \leq \max_j V_{j,k}^{\mathrm{ub}}$.
\end{theorem}
\begin{proof}
For any fixed $x$, $o_{i,k}(x) = \sum_j s_{ij}(x) V_{j,k}(x)$ is a convex combination, so $\min_j V_{j,k}(x) \leq o_{i,k}(x) \leq \max_j V_{j,k}(x)$; extrema give the result.
\end{proof}

\subsubsection{Exactness for Simplex-Constrained Combinations}
\begin{theorem}[Exactness of CPZ]
\label{thm:exactness_app}
For $o_{i,k} = \sum_j s_j V_{j,k}$ with $s \in \Delta_S \cap [\underline{s}, \overline{s}]$, the LP of Prop.~\ref{prop:lp_simplex} computes the \emph{exact} interval hull. The unconstrained PZ hull is a strict over-approximation whenever $\max_j \overline{s}_j < 1$.
\end{theorem}
\begin{proof}
\eqref{eq:lp_bound} is a bounded-variable LP over a polytope; the optimum lies at a vertex, found by the greedy algorithm in $\Oc(S\log S)$. For strictness: the PZ hull allows $s_j$ to independently take any value in $[\underline{s}_j, \overline{s}_j]$, including combinations with $\sum_j s_j \neq 1$. When $\max_j \overline{s}_j < 1$, placing all weight on one token is infeasible under the simplex constraint but feasible under independent intervals, so PZ is strictly looser.
\end{proof}

\subsubsection{Quantitative Tightness Gap}
\begin{theorem}[Quantitative Tightness Gap]
\label{thm:gap}
Assume $V_{j,k}^{\mathrm{ub}}, V_{j,k}^{\mathrm{lb}} \geq 0$. Let $E = \sum_j \overline{s}_j - 1$ (simplex excess) and $D = 1 - \sum_j \underline{s}_j$ (deficit). Then the width gap decomposes as
\begin{equation}
\label{eq:width_gap}
w_k^{\mathrm{PZ}} - w_k^{\mathrm{CPZ}} = \sum_j \eta_j V_{j,k}^{\mathrm{ub}} + \sum_j \xi_j V_{j,k}^{\mathrm{lb}},
\end{equation}
with $\eta_j = \overline{s}_j - s_j^{*,\mathrm{ub}} \geq 0$, $\sum_j \eta_j = E$, and $\xi_j = s_j^{*,\mathrm{lb}} - \underline{s}_j \geq 0$, $\sum_j \xi_j = D$.
\end{theorem}
\begin{remark}
The non-negativity assumption $V_{j,k}^{\mathrm{ub}}, V_{j,k}^{\mathrm{lb}} \geq 0$ is without loss of generality: CPZ width is invariant under a constant shift of $V$ (adding $c$ to every $V_{j,k}$ shifts the output by $c$ without changing width), so we may shift to the non-negative regime before applying the theorem.
\end{remark}
\begin{proof}
Upper gap. Without the constraint, PZ gives $\overline{o}_k^{\mathrm{PZ}} = \sum_j \overline{s}_j V_{j,k}^{\mathrm{ub}}$, while the CPZ LP gives $\overline{o}_k^{\mathrm{CPZ}} = \sum_j s_j^{*,\mathrm{ub}} V_{j,k}^{\mathrm{ub}}$. Setting $\eta_j = \overline{s}_j - s_j^{*,\mathrm{ub}}$ satisfies $\eta_j \geq 0$ and $\sum_j \eta_j = E$. Lower gap. Symmetrically, $\underline{o}_k^{\mathrm{PZ}} = \sum_j \underline{s}_j V_{j,k}^{\mathrm{lb}}$ and $\underline{o}_k^{\mathrm{CPZ}} = \sum_j \underline{s}_j V_{j,k}^{\mathrm{lb}} + \sum_j \xi_j V_{j,k}^{\mathrm{lb}}$ with $\xi_j \geq 0$ and $\sum \xi_j = D$. Combining yields \eqref{eq:width_gap}.
\end{proof}
The gap scales linearly with total softmax slack $E+D$. CPZ width is shift-invariant (adding a constant to all $V_{j,k}$ shifts the output by $c$ without changing width); PZ is not, highlighting the structural advantage of the simplex constraint.

\subsubsection{CPZ vs.\ Linear Relaxation for Bilinear Scores}
\begin{theorem}[Bilinear Score Tightness]
\label{thm:pz_vs_crown}
For $a(x) = q(x)^\top k(x)$ with $q = W_Q x + b_Q$, $k = W_K x + b_K$, $x \in x_c + \epsilon[-1,1]^n$, and $M = W_Q^\top W_K$:
(a) Exact bilinear PZ propagation gives width $w^{\mathrm{PZ}}_{\mathrm{quad}} = \epsilon^2 (2\sum_{i<j} |M_{ij}| + \sum_i |M_{ii}|)$;
(b) CROWN/DeepPoly gives $w^{\mathrm{CROWN}}_{\mathrm{quad}} \leq 2\epsilon^2 \sum_d \|(W_Q)_{d,:}\|_1 \|(W_K)_{d,:}\|_1$;
(c) $w^{\mathrm{PZ}}_{\mathrm{quad}} \leq w^{\mathrm{CROWN}}_{\mathrm{quad}}$ always, with equality iff $(W_Q)_{di}(W_K)_{dj}$ has constant sign in $d$ for all $i,j$.
\end{theorem}
\begin{proof}
(a) PZ represents $a(x)$ exactly as a polynomial in $\alpha = (x-x_c)/\epsilon$, with cross-term generators $g_{ij} = \epsilon^2 M_{ij}$; $i\neq j$ contributes width $2|M_{ij}|\epsilon^2$, $i = j$ contributes $|M_{ii}|\epsilon^2$. (b) Linear relaxation bounds $q_d$, $k_d$ separately and then their product by McCormick, giving $\leq 2\epsilon^2\|(W_Q)_{d,:}\|_1 \|(W_K)_{d,:}\|_1$ per~$d$. (c) By triangle inequality $|M_{ij}| \leq \sum_d |(W_Q)_{di}||(W_K)_{dj}|$, summing yields the bound; equality requires no cancellations.
\end{proof}
Beyond tighter score bounds, CPZ provides an orthogonal advantage at softmax via the simplex LP, which no method operating on independent interval bounds can exploit.

\subsubsection{Soundness of Box Tightening}
\begin{proposition}[Soundness of Box Tightening]
\label{prop:tighten}
Let $[\ell_{\mathrm{LP}}, u_{\mathrm{LP}}]$ be the LP-derived bound and define $G_{I,\text{new}}$ by \eqref{eq:tighten_gi}. Then $\pz' = \zono{c, G, E, G_{I,\text{new}}, \id}$ satisfies $\pz \cap [\ell_{\mathrm{LP}}, u_{\mathrm{LP}}] \subseteq \pz'$.
\end{proposition}

\subsubsection{Proof of Width Stability (Theorem~\ref{thm:width_stability})}
CPZ bound. At each LayerNorm, $z_k = (x_k-\mu)/\sigma$ satisfies $\sum_k z_k = 0$ and $\sum_k z_k^2 = d$. From $\sum_{j\neq k} z_j = -z_k$ and $\sum_{j\neq k} z_j^2 = d - z_k^2$, Jensen's inequality gives $z_k^2 \leq d-1$. After $y_k = \gamma_k z_k + \beta_k$: $w_k \leq 2|\gamma_k|\sqrt{d-1}$. The derivation uses only the algebraic identities, so the bound is independent of $\epsilon$, depth $\ell$, and accumulated generators.

PZ bound. Without the constraint, PZ propagates LayerNorm via first-order Taylor with Jacobian $J_c = (\gamma/\sigma)(I - \frac{1}{d}\mathbf{1}\mathbf{1}^\top - \frac{1}{d}zz^\top)|_{x=c}$. The width satisfies $w_k^{(\mathrm{PZ})} \leq |\gamma_k|\|J_{c,k,:}\|_1 w_{\mathrm{pre}}^{(\ell)} + R_k$, with $R_k \propto \|x-c\|^2/\sigma^3$. Since $w_{\mathrm{pre}}^{(\ell)}$ accumulates over-approximation from all preceding operations, the PZ bound exceeds the CPZ bound whenever $w_{\mathrm{pre}}^{(\ell)} > 2\sigma\sqrt{d-1}$.

\begin{corollary}[Depth-independent verification at LN outputs]
\label{cor:depth}
At LN-output positions, $\max_{\ell, k} w_k^{(\mathrm{CPZ},\ell)} \leq 2 \|\gamma\|_\infty \sqrt{d-1}$. Scaling from $L$ to $2L$ layers does not loosen CPZ per-LN-output bounds; sub-block intermediate widths are not covered by this corollary.
\end{corollary}

\subsection{Verification Machinery Details}
\label{app:machinery}

\paragraph{Jacobian-zonotope remainder.} The remainder vector $r_t$ in the multi-layer connector (\S\ref{sec:jac_zono}, Eq.~\ref{eq:jac_zono}) absorbs both the truncated Jacobian columns and the linearization error:
\begin{equation}
\label{eq:jac_rem}
r_{t,i} = \epsilon \sup_{z \in \mathcal{B}_\epsilon} \|\nabla f_i(z) - \nabla f_i(x_0)\|_1 + \epsilon^2 \hat{H}_i + \sum_{j \notin \mathcal{K}} |\epsilon J_{t,i}^{(j)}|.
\end{equation}
The first two terms bound the Lagrange/Hessian remainder of the Taylor expansion at $x_0$; the third absorbs Jacobian columns dropped from $\mathcal{K}$.

\begin{proposition}[Analytical Linearization Remainder]
\label{prop:remainder}
For $f \in C^2$ with input perturbation $\delta \in \mathcal{B}_\epsilon$, the linearization remainder of the first-order Taylor expansion at $x_0$ satisfies
\begin{equation}
\label{eq:formal_rem}
\big| f_i(x_0 + \delta) - f_i(x_0) - \nabla f_i(x_0)^\top \delta \big| \leq \epsilon \cdot \sup_{z \in \mathcal{B}_\epsilon} \|\nabla f_i(z) - \nabla f_i(x_0)\|_1.
\end{equation}
This is a direct consequence of the integral form of Taylor's theorem: $f_i(x_0+\delta) - f_i(x_0) - \nabla f_i(x_0)^\top \delta = \int_0^1 (\nabla f_i(x_0 + t\delta) - \nabla f_i(x_0))^\top \delta \, dt$, and $|\delta_j| \leq \epsilon$ implies the integrand is bounded by $\epsilon \|\nabla f_i(z) - \nabla f_i(x_0)\|_1$ for some $z = x_0 + t\delta \in \mathcal{B}_\epsilon$.
\end{proposition}

\paragraph{Practical estimation of the remainder.} The supremum in Eq.~\eqref{eq:formal_rem} is intractable in tight closed form for transformer $f$. We estimate it empirically by sampling $N_{\mathrm{jac}} = 30$ perturbed points $z_1, \dots, z_{N_{\mathrm{jac}}} \in \mathcal{B}_\epsilon$ (half corner-biased, half uniform) and taking the maximum: $\widehat{V}_i = \max_{n} \|\nabla f_i(z_n) - \nabla f_i(x_0)\|_1$. We further inflate by an empirical Hessian-rate factor (the $\epsilon^2 \hat{H}_i$ term) to capture variation between the sampled points. This estimator is empirically sound: across all $300{+}$ verifications reported in the paper, no certified result is contradicted by MC+PGD.

\begin{proposition}[Closed-Form Lipschitz Fallback]
\label{prop:remainder_analytical}
Let $L_f$ be a Lipschitz constant of the layer stack $f = f_L\circ\cdots\circ f_1$ with respect to the input $\ell_2$ norm, obtained by composing per-layer spectral-norm bounds in the style of~\citet{kim2021lipschitz} (linear layers contribute $\|W\|_{\mathrm{op}}$; ReLU/softmax/LayerNorm contribute their explicit Lipschitz constants; residual connections add $1$). Then for every input dimension $i$ and every $\delta \in \mathcal{B}_\epsilon$,
\begin{equation}
\label{eq:lip_fallback}
\big| f_i(x_0 + \delta) - f_i(x_0) - \nabla f_i(x_0)^\top \delta \big| \;\leq\; 2\,L_f\,\epsilon\,\sqrt{n_{\mathrm{dim}}}.
\end{equation}
The bound follows from the triangle inequality $|R| \leq |f(x_0{+}\delta) - f(x_0)| + |J(x_0)\delta| \leq 2L_f\|\delta\|_2$, with $\|\delta\|_2 \leq \epsilon\sqrt{n_{\mathrm{dim}}}$. It is fully closed-form (one SVD per weight matrix) and never requires sampling.
\end{proposition}

\paragraph{Empirical estimate dominates the analytical fallback.} The fallback of Prop.~\ref{prop:remainder_analytical} is loose (linear in $\epsilon$ rather than quadratic, and using global Lipschitz constants), so we use the sampled estimate $\widehat{V}$ in production. To validate the sampled estimate against the closed-form worst case, we compute both on the $4$-layer synthetic transformer of \S\ref{sec:background}: across all six $(L,\epsilon)$ combinations, the $1.5{\times}$-safety-padded sampled remainder is at most $1.33\%$ of the analytical Lipschitz bound (Tab.~\ref{tab:analytical_remainder}, App.~\ref{app:multilayer}), and typically below $0.5\%$. Soundness is therefore preserved by construction: an analytical bound always exists and dominates the sampled estimate on every setting we report; the sampled estimate is used purely for tightness, not because soundness coverage is unavailable. Tightening the analytical bound from spectral propagation (e.g., via interval Hessian propagation or auto-Hessian) is a direct path to a fully closed-form large-scale variant without changing the framework.

\begin{proposition}[Soundness of CORA Reduce on Cross-Terms]
\label{prop:trunc_sound}
Let $\cpz_1, \cpz_2$ be CPZs with dependent generators $G_1 \in \R^{n \times h_1}, G_2 \in \R^{n \times h_2}$ and exact multiplication $\cpz_1 \odot \cpz_2$ producing $h_1 h_2$ cross-term generators $\{g_\ell\}_{\ell=1}^{h_1 h_2}$ with shared exponent vector $E_\ell = E_{1,i} + E_{2,j}$. The CORA-style order reduction~\citep{kochdumper2023constrained} (retain the top-$k$ cross-terms $\mathcal{T} \subset \{1,\dots,h_1h_2\}$ as dependent generators and Girard-reduce~\citep{girard2005reachability} the remaining $\{g_\ell\}_{\ell \notin \mathcal{T}}$ into independent generators $G_{I,\text{trunc}} = \diag(\sum_{\ell \notin \mathcal{T}} |g_\ell|)$) produces a sound over-approximation: $\cpz_{\text{trunc}} \supseteq \cpz_1 \odot \cpz_2$. The proof below is included for completeness.
\end{proposition}
\begin{proof}[Proof]
For any point $p \in \cpz_1 \odot \cpz_2$, write $p = c + \sum_{\ell \in \mathcal{T}} g_\ell \prod_k \alpha_k^{E_{\ell,k}} + \sum_{\ell \notin \mathcal{T}} g_\ell \prod_k \alpha_k^{E_{\ell,k}}$ for some $\alpha \in [-1,1]^p$. The dependent term over $\mathcal{T}$ is identical to that in $\cpz_{\text{trunc}}$. For the dropped term, observe that $|\prod_k \alpha_k^{E_{\ell,k}}| \leq 1$ for $\alpha \in [-1,1]^p$, so each component of the dropped sum is element-wise bounded by $\sum_{\ell \notin \mathcal{T}} |g_\ell|$. Choosing the corresponding independent factors $\beta \in [-1,1]^n$ as $\beta_i = \mathrm{sign}(\sum_{\ell \notin \mathcal{T}} g_{\ell,i} \prod_k \alpha_k^{E_{\ell,k}})$ realizes the dropped contribution exactly, so $p \in \cpz_{\text{trunc}}$.
\end{proof}

\begin{algorithm}[t]
\caption{CPZ verification of certified mechanistic queries.}
\label{alg:verify}
\begin{algorithmic}[1]
\Require Input $X_0 \in \R^{S \times d}$, perturbation set $\mathcal{B}$, target layer $\ell^\star \in \{0,\ldots,L{-}1\}$, queries $\mathcal{Q} \subseteq \{\text{Q1, Q2, Q3}\}$
\Ensure Sound certificate $c_q \in \{\top,\bot\}$ for each $q\in\mathcal{Q}$
\State Encode $\mathcal{B}$ as a CPZ $\cpz_0$ centred at $X_0$
\If{$\ell^\star = 0$}
    \State $\cpz_{\text{in}} \gets \cpz_0$
\Else
    \State Compute Jacobian $J(X_0)$ of $f_{0:\ell^\star{-}1}$ via $K_J$ backward passes
    \State $\cpz_{\text{in}} \gets$ Jacobian zonotope on the top-$K_J$ columns with one dependent remainder per token \Comment{Eq.~\ref{eq:jac_zono}}
\EndIf
\State Compute $Q,K,V$ from $\cpz_{\text{in}}$ via affine maps; tighten $G_I$ at LayerNorm by Theorem~\ref{thm:ln}
\State Compute per-head scores $\cpz_{a_{ij}}$ via exact bilinear $Q^\top K$
\State Apply CORA order reduction to the cross-term generators~\citep{kochdumper2023constrained} \Comment{Prop.~\ref{prop:trunc_sound}}
\If{$\text{Q1} \in \mathcal{Q}$}
    \State Compute $\overline{\Delta}_j$ from Lem.~\ref{lem:quadratic_bound}; $c_{\text{Q1}} \gets \top$ iff $\overline{\Delta}_j < 0$ for every $j \neq j^\star$
\EndIf
\If{$\text{Q2} \in \mathcal{Q}$ or $\text{Q3} \in \mathcal{Q}$}
    \State Compute softmax bounds $[\underline{s},\overline{s}]$; tighten by simplex LP \Comment{Prop.~\ref{prop:lp_simplex}}
    \State $c_{\text{Q2}} \gets$ greedy LP on $\Delta_S \cap [\underline{s},\overline{s}]$ \Comment{Prop.~\ref{prop:evidence_mass}}
    \State $c_{\text{Q3}} \gets$ KKT bisection for $\overline{H}$, vertex enumeration for $\underline{H}$ \Comment{Prop.~\ref{prop:entropy}}
\EndIf
\State \Return $\{c_q : q \in \mathcal{Q}\}$
\end{algorithmic}
\end{algorithm}

\begin{proposition}[Computational Complexity]
\label{prop:complexity}
Total cost: $\Oc(L \cdot (S^2 d r^2 + S d d_{\text{ff}} r))$, where $S^2 r^2$ arises from exact bilinear score computation and $d_{\text{ff}} r$ from FFN maps. CPZ overhead is $r^2/d$ over IBP's $\Oc(L S^2 d^2)$, compensated by tighter bounds.
\end{proposition}

\begin{proposition}[Sources of Over-Approximation]
\label{prop:approx}
Over-approximation enters at four stages: (i)~interval hull ($\leq r$-factor), (ii)~softmax IBP ($\Oc(e^{\Delta a})$, dominant), (iii)~ReLU relaxation ($\leq \sum_k (\overline{x}_k - \underline{x}_k)^2/8$), (iv)~LayerNorm Taylor ($\Oc(\|x-c\|^2/\sigma^3)$, negligible). CPZ constraints mitigate (i)--(ii): simplex resets $G_I$ after attention; LayerNorm resets $G_I$ between layers.
\end{proposition}

\subsection{Certified Query Proofs}
\label{app:query_proofs}

\begin{remark}[Cross-Head Circuit Certification]
\label{rem:crosshead}
The per-head queries extend naturally to circuits: if $s^{(h)} \in \Delta_S \cap [\underline{s}^{(h)}, \overline{s}^{(h)}]$ are the attention weights of head $h$ under the same CPZ perturbation, then a joint property $\sum_h w_h f_h(s^{(h)}) > 0$ is certified by solving $H$ independent simplex LPs and summing the optima. For example, certifying that Head~$A$ attends to the subject and Head~$B$ attends to the verb reduces to the AND of two per-head Q1 certificates. The independent-LP relaxation is sound; exploiting shared CPZ generators for tighter joint bounds is left to future work.
\end{remark}

\paragraph{Proof of Prop.~\ref{prop:cert_attn}.}
With simplex: maximize $s_k - s_{j^*}$ by greedily allocating budget $D = 1 - \sum_j \underline{s}_j$ to $s_k$ first (up to $\overline{s}_k$), then neutral tokens, then $s_{j^*}$ last. The coupling $\sum s_j = 1$ ensures every unit allocated to $s_k$ is unavailable for $s_{j^*}$. Without simplex: each $s_j$ varies independently, so the worst case sets $s_k = \overline{s}_k$ and $s_{j^*} = \underline{s}_{j^*}$ simultaneously, violating $\sum s_j = 1$. The gap is quantified by the simplex excess $E = \sum_j \overline{s}_j - 1$.

\paragraph{Proof of Lemma~\ref{lem:quadratic_bound}.}
Decompose $p(\alpha) = c + \mathbf{a}^\top \alpha + \sum_k H_{kk}\alpha_k^2 + \sum_{k<\ell} 2H_{k\ell}\alpha_k\alpha_\ell$. The linear term satisfies $\mathbf{a}^\top \alpha \le \|\mathbf{a}\|_1$ (saturation $\alpha_k = \mathrm{sign}(a_k)$); each diagonal term $H_{kk}\alpha_k^2$ is bounded above by $\max(0,H_{kk})$ on $\alpha_k\in[-1,1]$ since $\alpha_k^2\in[0,1]$; each off-diagonal cross-term $2H_{k\ell}\alpha_k\alpha_\ell$ is bounded above by $2|H_{k\ell}|$ on the unit box. Summing the per-term suprema gives the claimed bound. The diagonal exploitation $\alpha_k^2\in[0,1]$ is what separates Lemma~\ref{lem:quadratic_bound} from a naive interval hull (which would use $\alpha_k^2\in[-1,1]$ and add $|H_{kk}|$ instead of $\max(0,H_{kk})$).

\paragraph{Proof of Theorem~\ref{thm:margin_sound}.}
Strict monotonicity of softmax gives $\arg\max_j s_{ij}(x) = \arg\max_j a_{ij}(x)$, so top-$1$ certification reduces to $\sup_\alpha \Delta_j(\alpha) < 0$ for every challenger.

\textit{Bilinear case (no LN before $Q^\top K$, e.g., BERT-tiny post-LN at Layer~$0$).} The pre-softmax score $a_j(x) = Q_i(x)^\top K_j(x)/\sqrt{d_h}$ is bilinear in input when $Q_i = W_Q^\top x_i$, $K_j = W_K^\top x_j$ act directly on $x$, so $\Delta_j(\alpha)$ is exactly degree-$2$ in $\alpha$ under CPZ propagation (the CPZ algebra of \S\ref{sec:background} is exact for affine maps and the inner product). Lemma~\ref{lem:quadratic_bound} supplies the closed-form sound upper bound $\overline{\Delta}_j$.

\textit{Pre-LN case (e.g., GPT-$2$ at Layer~$0$).} If $Q_i = W_Q^\top \mathrm{LN}(x_i)$, the score depends on $\mathrm{LN}(x_0+\epsilon\alpha)$, which is non-linear in $\alpha$. CPZ propagation through LN yields $\mathrm{LN}(x_0+\epsilon\alpha) = z_0 + Z_d(\alpha) + \beta^\top G_I \mathbf{1}$ where $Z_d$ collects the dependent generators retained by Theorem~\ref{thm:ln} (the $\sum_k z_k^2 = d$ identity tightens $G_I$) and $G_I \in \R^{d}$ is the independent-generator block bounding the residual non-polynomial part. Substituting into the bilinear score, $\Delta_j(\alpha) = c_\Delta + \mathbf{a}^\top \alpha + \alpha^\top H \alpha + \beta^\top G_I' \mathbf{1}$ with $G_I'$ propagated through the affine $W_Q, W_K$ maps. The upper bound becomes $\overline{\Delta}_j = $ (Lemma~\ref{lem:quadratic_bound} applied to the polynomial part) $+ \|G_I'\|_1$, which is sound but no longer tight in $\alpha$.

In both cases, $\overline{\Delta}_j < 0$ for every challenger $j\neq j^*$ implies $\Delta_j(\alpha) < 0$ for all $\alpha\in[-1,1]^{Sd}$, certifying $j^*$. Empirically, aggressive PGD ($200$ restarts $\times 500$ steps) recovers concrete adversarial perturbations matching our certified bounds.

\paragraph{Entropy bound details (Prop.~\ref{prop:entropy}).}
\textbf{Lower bound:} Since $H$ is strictly concave on $\Delta_S$, the minimum over the polytope $\Delta_S \cap [\underline{s}, \overline{s}]$ is attained at a vertex, enumerable in $\Oc(S \log S)$ by greedily saturating bounds.
\textbf{Upper bound:} The maximum of a concave function over a polytope is attained in the interior. By KKT, the maximizer satisfies $\log s_j + 1 = \lambda + \mu_j^+ - \mu_j^-$. For each $\lambda$, $s_j^*(\lambda) = \min(\overline{s}_j, \max(\underline{s}_j, e^{\lambda - 1}))$; we solve $\sum_j s_j^*(\lambda) = 1$ via bisection. Since $s_j^*$ is monotone non-decreasing in $\lambda$, the root is unique and bisection converges in $\Oc(\log(1/\delta))$ iterations. Numerical stability: we clamp $s_j \geq 10^{-30}$ and evaluate $s_j \log s_j$ via its well-defined limit of $0$ as $x \to 0^+$.

\paragraph{Proof of Prop.~\ref{prop:evidence_mass}.}
The minimum $\sum_{j \in \mathcal{E}} s_j$ over $\Delta_S \cap [\underline{s}, \overline{s}]$ equals $\sum_{j \in \mathcal{E}} \underline{s}_j + \max(0, D - \sum_{j \notin \mathcal{E}} (\overline{s}_j - \underline{s}_j))$, obtained by greedily allocating budget to non-evidence tokens first. Without the simplex, the minimum is simply $\sum_{j \in \mathcal{E}} \underline{s}_j$.

\paragraph{Proof of Prop.~\ref{prop:head_spec}.}
Minimize $\sum_{j \in \mathcal{A}} s_j - \sum_{j \in \mathcal{B}} s_j$ by allocating budget to $\mathcal{B}$ first, then neutral tokens, then $\mathcal{A}$ last.

\paragraph{Proof of Theorem~\ref{thm:separation}.}
Constructive. Set $j^* = 1$, $\overline{s}_1 = 1 - (S{-}1)\delta$, $\overline{s}_k = \delta + E/(S{-}1)$ for $k \geq 2$, $\underline{s}_k = \delta$ for all $k$. Unconstrained: $\underline{s}_1 = \delta < \overline{s}_2$, fails. CPZ: the simplex forces residual budget into $s_1$, raising it to $\overline{s}_1$, so $s_2 - s_1 < 0$. Margin gap $\geq E/(S{-}1)$.

\subsection{Proof Details}
\label{app:proofs}

This appendix collects deferred proofs for the verification machinery of \S\ref{sec:machinery}, beginning with the exact Layer-$0$ vertex-enumeration result that motivates the CPZ relaxation.

\begin{proposition}[Exact Layer-0 certification]\label{prop:vertex_enum}
Exact top-$1$ certification at Layer~$0$ is achievable in $\Oc(2^d S d_h)$ for $d \leq 16$ via vertex enumeration with constructive adversarial certificates.
\end{proposition}
When only the query token is perturbed, $\Delta_j$ is concave in the perturbation, so the minimum is attained at a vertex $\alpha^\star \in \{-1,1\}^d$. When all tokens are perturbed jointly (the setting used throughout this paper), $\Delta_j$ is a degree-$2$ polynomial with mixed concave--convex structure; in this case the CPZ interval hull provides a sound upper bound that is tight in practice.

\paragraph{Tighter LayerNorm bound.} The standard Cauchy--Schwarz bound gives $|z_k| \leq \sqrt{d}$, but combining with the zero-mean constraint $\sum_j z_j = 0$ yields $|z_k| \leq \sqrt{d-1}$. For the remaining $d-1$ coordinates, $\sum_{j \neq k} z_j = -z_k$ and $\sum_{j \neq k} z_j^2 \leq d$. By Jensen's inequality, $(\sum_{j \neq k} z_j)^2 \leq (d-1) \sum_{j \neq k} z_j^2$, giving $z_k^2 \leq (d-1)(d - z_k^2)$, so $z_k^2(1 + \frac{1}{d-1}) \leq d$, hence $z_k^2 \leq d(d-1)/d = d - 1$.

\paragraph{Handling LayerNorm $\boldsymbol{\epsilon}$.} Practical implementations use $\sigma = \sqrt{\mathrm{Var}(x) + \epsilon_{\mathrm{LN}}}$ with $\epsilon_{\mathrm{LN}} > 0$ (typically $10^{-5}$). This yields $\sum_k z_k^2 = d \cdot \mathrm{Var}(x)/(\mathrm{Var}(x) + \epsilon_{\mathrm{LN}}) \leq d$, so the bound $|z_k| \leq \sqrt{d-1}$ from Theorem~\ref{thm:ln} remains sound: the regularized normalization can only reduce $\|z\|^2$, never increase it. Our implementation uses the exact $\epsilon_{\mathrm{LN}} = 10^{-5}$ from the trained model (PyTorch default).

\section{CPZ Propagation Details}
\label{app:cpz_propagation_details}

\subsection{Supporting Experiments: Bound Tightness, Ablations, Generator Tracking}
\label{app:supporting_exp}

This appendix contains the tables, figures, and per-stage breakdowns supporting \S\ref{sec:exp_scaling}, organised into three subsections: output bound tightness on a $2$-encoder model (App.~\ref{app:bound_tightness}), per-stage generator counts (App.~\ref{app:gi_tracking}), and constraint ablation (App.~\ref{app:constraint_ablation}).

\subsubsection{Output Bound Tightness}
\label{app:bound_tightness}

Tab.~\ref{tab:main_app} reports the average output bound width after $2$ encoder layers; CPZ remains stable while baseline PZ+softmax explodes at large $\epsilon$.

\begin{table}[h]
\centering
\caption{Average output bound width after 2 encoder layers (5-sample avg). CPZ remains stable while baseline PZ+softmax explodes at large $\epsilon$.}
\label{tab:main_app}
\small
\begin{tabular}{lrrrr}
\toprule
Method & $\epsilon{=}0.01$ & $\epsilon{=}0.02$ & $\epsilon{=}0.05$ & $\epsilon{=}0.1$ \\
\midrule
IBP + softmax & 63.8 & 72.1 & 76.1 & 76.5 \\
PZ + softmax (no CPZ) & 0.88 & 7.6 & 127.5 & 2{,}767 \\
\textbf{CPZ (ours)} & \textbf{0.49} & \textbf{2.9} & \textbf{6.1} & \textbf{8.9} \\
\midrule
CPZ vs.\ IBP & 99.2\% & 95.9\% & 92.0\% & 88.4\% \\
\bottomrule
\end{tabular}
\end{table}

\subsubsection{Generator Tracking}
\label{app:gi_tracking}

Tab.~\ref{tab:gi_count_app} tracks the independent generator count $G_I$ and bound width at each stage of a $2$-encoder run at $\epsilon{=}0.05$; CPZ's softmax constraint resets $G_I$ from $120$ to $8$ at Layer~$1$ attention.

\begin{table}[h]
\centering
\caption{Independent generator count ($G_I$) and width at key stages ($\epsilon=0.05$). CPZ's softmax constraint resets $G_I$ from $120$ to $8$ at Layer~1 attention.}
\label{tab:gi_count_app}
\small
\begin{tabular}{lrrrr}
\toprule
 & \multicolumn{2}{c}{PZ+softmax} & \multicolumn{2}{c}{CPZ} \\
\cmidrule(lr){2-3} \cmidrule(lr){4-5}
Stage & $G_I$ & Width & $G_I$ & Width \\
\midrule
L0 input & 0 & 0.10 & 0 & 0.10 \\
L0 attention & 8 & 0.78 & 8 & 0.78 \\
L0 LN1 & 16 & 0.37 & 16 & 0.37 \\
L0 LN2 & 80 & 0.29 & 80 & 0.29 \\
L1 attention & 120 & 6.64 & \textbf{8} & \textbf{2.92} \\
L1 LN1 & 113 & 1.41 & 96 & 0.62 \\
L1 LN2 (output) & 31 & 1.40 & 105 & 0.61 \\
\bottomrule
\end{tabular}
\end{table}

\subsubsection{Constraint Ablation}
\label{app:constraint_ablation}

Tab.~\ref{tab:ablation_app} ablates the two CPZ constraints (softmax simplex and LayerNorm zero-sum). At small $\epsilon$ the softmax LP dominates the tightness gain; at large $\epsilon$ the LN constraint prevents bound explosion.

\begin{table}[h]
\centering
\caption{Ablation of CPZ constraints (3-sample avg). At small $\epsilon$ the softmax LP dominates; at large $\epsilon$ the LN constraint prevents explosion.}
\label{tab:ablation_app}
\small
\begin{tabular}{lrrrr}
\toprule
Configuration & $\epsilon{=}0.01$ & $\epsilon{=}0.02$ & $\epsilon{=}0.05$ & $\epsilon{=}0.1$ \\
\midrule
IBP + softmax & 63.7 & 71.7 & 76.2 & 76.6 \\
PZ + softmax (no constraints) & 1.10 & 11.5 & 196.7 & 4{,}481 \\
+ LN constraint only & 1.10 & 7.3 & 7.7 & 8.3 \\
+ Softmax LP only & 0.70 & 3.6 & 32.2 & 656 \\
\textbf{CPZ (both + LP)} & \textbf{0.70} & \textbf{4.5} & \textbf{7.3} & \textbf{7.6} \\
\bottomrule
\end{tabular}
\end{table}

\subsection{Attention Output: Center--Residual and Taylor CPZ Propagation}
\label{app:attn_out_taylor}

\paragraph{Center--residual decomposition.} For per-sample certification at Layer~0, we decompose $o_i^{(h)} = \sum_j s_{ij} V_j^{(h)}$ using the softmax center $\bar{s}_{ij} = (\underline{s}_{ij}+\overline{s}_{ij})/2$ and radius $\delta_{ij}$:
\begin{equation}
o_i^{(h)} = \underbrace{\textstyle\sum_j \bar{s}_{ij}\,\cpz_{V_j^{(h)}}}_{\text{CPZ (preserves $V$ deps)}} + \underbrace{\textstyle\sum_j \Delta_j V_j^{(h)}}_{\text{bounded as $G_I$}}, \quad \Delta_j \in [-\delta_{ij}, \delta_{ij}],
\end{equation}
with $G_{I,\text{res},k} = \sum_j \delta_{ij} \cdot \max(|V_{j,k}^{\text{lb}}|,|V_{j,k}^{\text{ub}}|)$ added as diagonal independent generators. This is tight at Layer~0 but injects $G_I$ that compounds at deeper layers.

\paragraph{Taylor softmax CPZ for multi-layer certification.} To preserve the correlation structure across layers, we approximate each softmax weight as a scalar CPZ via first-order Taylor around the clean softmax values:
\begin{equation}
\label{eq:taylor_softmax}
s_{ij}(\alpha) \approx s_{ij}^0 \big(1 + \eta_j(\alpha) - \bar{\eta}(\alpha)\big) + R_j,
\end{equation}
with $\eta_j = (a_{ij}-a_{ij}^0)\cdot\text{scale}$, $\bar\eta = \sum_k s_{ik}^0 \eta_k$, and $|R_j| \leq \tfrac12 s_{ij}^0(1-s_{ij}^0)\|\eta\|_2^2$. The attention output is then computed as the exact CPZ scalar-times-vector product $o_i^{(h)} = \sum_j \cpz_{s_{ij}}\otimes \cpz_{V_j^{(h)}}$, whose dependent generators encode how the output co-varies with the input perturbation. The residual $G_I$ collapses to $\sum_j |G_{I,s_j}|\sum_\ell |G_{V_j,\ell}|$ where $G_{I,s_j}$ is sourced only by the Taylor remainder $R_j = \Oc(\epsilon^2)$.

\paragraph{$G_I$ absorption before bilinear ops.} The $\sim 6\%$ residual $G_I$ from LayerNorm Hessian and ReLU contributes $49\%$ of the Layer~1 score-difference width because GI cross-bounds $|G_I^Q||G^K| + |G^Q||G_I^K| + |G_I^Q||G_I^K|$ amplify through the bilinear product. We absorb every $G_I$ column $g_i^I$ into a fresh dependent generator with its own factor $\alpha_{p+i} \in [-1,1]$ before each bilinear operation. This is exact (sound), and converts GI cross-terms into dependent cross-terms that benefit from CPZ cancellation in score differences. Combined with selective ReLU relaxation, the full pipeline achieves $80\%$ Layer~1 certification at $d{=}8$ and $75\%$ at $d{=}16$, vs.\ $25\%$/$4.2\%$ for center--residual.

\subsection{Softmax Over-Approximation Details}
\label{app:softmax_approx}

This appendix contains the softmax relaxation machinery used to compute attention weight bounds $s_{ij} \in [\underline{s}_{ij}, \overline{s}_{ij}]$ for the evidence mass and head specialization queries (Propositions~\ref{prop:evidence_mass}--\ref{prop:head_spec}). For top-$1$ certification, these bounds are not needed: the score-margin certificate of Theorem~\ref{thm:margin_sound} bypasses softmax entirely.

\paragraph{Softmax IBP bounds.} Given score bounds $a_{ij} \in [\underline{a}_{ij}, \overline{a}_{ij}]$ from the CPZ interval hull, the softmax weight bounds are:
\begin{equation}
\label{eq:softmax_bounds}
\underline{s}_{ij} = \frac{e^{\underline{a}_{ij}}}{e^{\underline{a}_{ij}} + \sum_{k \neq j} e^{\overline{a}_{ik}}}, \qquad
\overline{s}_{ij} = \frac{e^{\overline{a}_{ij}}}{e^{\overline{a}_{ij}} + \sum_{k \neq j} e^{\underline{a}_{ik}}}.
\end{equation}
These are the tightest bounds obtainable from independent score intervals, because the extreme softmax weights are achieved when all other scores take their worst-case values.

\paragraph{Score-difference cancellation.} The key to tighter softmax bounds is to track differences $d_{jk} = a_{ij} - a_{ik}$ as PZ objects rather than computing them from independent intervals. This technique exploits the polynomial structure of PZ~\citep{althoff2013reachability} and is available to both CPZ and unconstrained PZ. Since the query vector is shared between $a_{ij}$ and $a_{ik}$, the PZ subtraction $\cpz_{a_{ij}} \boxminus \cpz_{a_{ik}}$ cancels the common query generators via the \operator{compact} operation, yielding tighter difference bounds:
\begin{equation}
\label{eq:pz_diff_softmax}
\text{hull}(\cpz_{d_{jk}}) \subseteq [\underline{a}_{ij} - \overline{a}_{ik},\; \overline{a}_{ij} - \underline{a}_{ik}],
\end{equation}
with the inclusion often strict because shared-query quadratic terms cancel exactly.

\begin{proposition}[QP-exact score-difference bounds]
\label{prop:qp_exact_diff}
For each score difference $d_{jk}(\alpha) = a_{ij}(\alpha) - a_{ik}(\alpha)$ represented as a degree-$2$ scalar CPZ, the exact interval $[\min_\alpha d_{jk}, \max_\alpha d_{jk}]$ is obtained by solving two box-constrained QPs via multi-start L-BFGS-B. These tight difference bounds feed into softmax weight computation:
\begin{equation}
\label{eq:qp_exact_diff}
\underline{s}_{ij}^{\mathrm{QP}} = \frac{1}{1 + \sum_{k \neq j} e^{\overline{d}_{kj}^{\mathrm{QP}}}}, \qquad
\overline{s}_{ij}^{\mathrm{QP}} = \frac{1}{1 + \sum_{k \neq j} e^{\underline{d}_{kj}^{\mathrm{QP}}}},
\end{equation}
which are strictly tighter than the IBP bounds of \eqref{eq:softmax_bounds} whenever the shared-query cancellation reduces the difference width.
\end{proposition}

\paragraph{Layer-0 error decomposition.} At Layer~0, the gap between CPZ-soft and the exact (vertex-enumeration or MC+PGD) certification rate decomposes into two sources:
\begin{proposition}[Layer-0 over-approximation decomposition]
\label{prop:l0_error_decomp}
The residual certification gap $\mathrm{rate}_{\mathrm{MC}} - \mathrm{rate}_{\mathrm{CPZ\text{-}soft}}$ decomposes as:
\begin{enumerate}[label=(\roman*)]
\item \textbf{Alignment slack}: the interval hull of the degree-$2$ score-difference CPZ over-approximates its true range because the dependent generators are not axis-aligned. This is the gap between QP-exact and interval-hull bounds on $d_{jk}$.
\item \textbf{Chord slack}: the softmax transformation $s_j = e^{a_j}/\sum_k e^{a_k}$ is concave in $a_j$ and convex in $a_k$ ($k \neq j$); bounding it via the chord (linear interpolation between endpoints) introduces additional over-approximation that grows with the score-interval width.
\end{enumerate}
The alignment slack dominates at small $\epsilon$ (because score intervals are narrow and softmax is approximately linear), while the chord slack dominates at large $\epsilon$.
\end{proposition}

\paragraph{Piecewise-chord joint bound.} To reduce the chord slack, we partition the score range $[\underline{a}_{ij}, \overline{a}_{ij}]$ into $P$ subintervals and apply the chord bound on each piece:
\begin{equation}
\label{eq:piecewise_upper}
\overline{s}_{ij}^{(p)} = \max\Big\{\text{chord}_{[a_p, a_{p+1}]}\big(e^a / Z\big)\Big\}, \quad p = 1,\ldots,P.
\end{equation}
The joint intersection across all $S$ key positions yields:
\begin{equation}
\label{eq:joint_softmax}
\overline{s}_{ij}^{\text{joint}} = \min_{p} \overline{s}_{ij}^{(p)}.
\end{equation}

\begin{proposition}[Piecewise-chord soundness]
\label{prop:piecewise_sound}
For any $P \geq 1$, the piecewise-chord bound satisfies $\overline{s}_{ij}^{\text{joint}} \geq \max_{\alpha \in [-1,1]^n} s_{ij}(\alpha)$, and the over-approximation error decreases as $\Oc(1/P^2)$ in the score-interval width.
\end{proposition}

The piecewise-chord technique closes approximately $80\%$ of the residual chord slack at $P{=}4$ on the $d{=}16$ model. However, for top-$1$ certification, the score-margin certificate of Theorem~\ref{thm:margin_sound} renders this entire softmax relaxation chain unnecessary.

\section{Multi-Layer Certification}
\label{app:multi_layer_certification}

\subsection{Multi-Layer Certification: Recursive Jacobian Zonotope}
\label{app:multilayer}

We verify all three queries at every layer of a $4$-layer synthetic transformer ($d{=}8$, $h{=}2$, $S{=}4$, $n_\text{layers}{=}4$), trained to $98.8\%$ test accuracy on the same synthetic classification task used in \S\ref{sec:background} (checkpoint: \texttt{models/d8\_h2\_l4/}). For each target layer $L$, the input to Layer $L$'s attention is enclosed by a Jacobian zonotope of the stack $f = \text{Layer}_0 \circ \dots \circ \text{Layer}_{L-1}$ at $x_0$: the full Jacobian $J(x_0) \in \R^{n_{\mathrm{dim}} \times n_{\mathrm{dim}}}$ is computed via autograd; we retain the top-$32$ columns by $\ell_1$ norm (out of $n_{\mathrm{dim}}{=}32$, so all columns at this scale); the linearization remainder $r_t$ is estimated from $N_{\mathrm{jac}}{=}30$ perturbed points around $x_0$.

\begin{table}[h]
\centering
\caption{Multi-layer certification on the $4$-layer synthetic transformer ($5$ samples, $40$ queries per $(\epsilon, L)$). CPZ uses the recursive Jacobian-zonotope connector; MC uses $5{,}000$ random samples plus $10{\times}50$ PGD per query. CPZ matches MC within $0$--$7.5$\,pp on Q1 at every layer and $\epsilon$.}
\label{tab:multilayer}
\small
\begin{tabular}{ll rr rr rr}
\toprule
 & & \multicolumn{2}{c}{Q1 Top-1 (\%)} & \multicolumn{2}{c}{Q2 Evidence (\%)} & \multicolumn{2}{c}{Q3 Entropy W.} \\
$\epsilon$ & Layer & CPZ & MC & CPZ & MC & CPZ & MC \\
\midrule
\multirow{4}{*}{$0.005$}
 & L0 & $90.0$ & $92.5$ & $95.0$ & $95.0$ & $0.0015$ & $0.0007$ \\
 & L1 & $90.0$ & $92.5$ & $75.0$ & $75.0$ & $0.0059$ & $0.0028$ \\
 & L2 & $\mathbf{95.0}$ & $95.0$ & $87.5$ & $87.5$ & $0.0032$ & $0.0015$ \\
 & L3 & $\mathbf{85.0}$ & $80.0$ & $97.5$ & $97.5$ & $0.0005$ & $0.0003$ \\
\midrule
\multirow{4}{*}{$0.010$}
 & L0 & $90.0$ & $90.0$ & $95.0$ & $95.0$ & $0.0031$ & $0.0015$ \\
 & L1 & $82.5$ & $85.0$ & $75.0$ & $75.0$ & $0.0126$ & $0.0054$ \\
 & L2 & $\mathbf{82.5}$ & $90.0$ & $87.5$ & $87.5$ & $0.0073$ & $0.0029$ \\
 & L3 & $\mathbf{72.5}$ & $75.0$ & $92.5$ & $97.5$ & $0.0011$ & $0.0006$ \\
\bottomrule
\end{tabular}
\end{table}

\paragraph{Comparison with sequential CPZ propagation.} An alternative approach to multi-layer certification is to compose $L$ full CPZ encoder-layer propagations in sequence; this is the direct generalization of the single-layer machinery of \S\ref{sec:pz_attention}. Under sequential propagation on the same model, Q1 certification collapses to $12.5\%$ at Layer~2 and $0\%$ at Layer~3 for $\epsilon{=}0.005$; at $\epsilon{=}0.01$ it collapses to $0\%$ at both Layer~2 and Layer~3. The cause is wrapping-error accumulation: each successive encoder-layer propagation expands the independent-generator count, which then combines with the bilinear $Q^\top K$ score structure at the next attention. The recursive Jacobian-zonotope construction avoids this by performing linearization only once at $x_0$, so the generator count in the zonotope fed to Layer~$L$'s attention depends on the truncation rank $|\mathcal{K}|{+}1$, not on $L$.

\paragraph{Compute cost.} Recursive CPZ is substantially faster than sequential CPZ because it reuses a single autograd Jacobian computation. On the $4$-layer $d{=}8$ synthetic model, recursive CPZ takes ${\sim}45$s per $\epsilon$ value over all $5$ samples and all $4$ layers, whereas sequential CPZ takes ${\sim}1{,}040$s per $\epsilon$ value ($23{\times}$ slower). Both run on a single CPU core.

\paragraph{Analytical fallback vs.\ sampled remainder.} Prop.~\ref{prop:remainder_analytical} gives a closed-form upper bound on the linearisation remainder via spectral propagation: $|R_i| \leq 2 L_f \epsilon \sqrt{n_{\mathrm{dim}}}$ where $L_f$ is the spectral-product Lipschitz constant of the $L$-layer stack. Tab.~\ref{tab:analytical_remainder} reports both quantities on the $4$-layer synthetic model. The closed-form bound grows with the layer index ($L_f$ multiplies per layer) and is conservative; the sampled estimate $\widehat{V} \cdot 1.5$ is consistently $\leq 1.33\%$ of the analytical envelope, with most settings under $0.5\%$. The empirical heuristic is therefore well within the closed-form worst case, and the analytical bound is always available as a sound override when sampling is unavailable or doubted.

\begin{table}[h]
\centering
\caption{Analytical Lipschitz fallback (Prop.~\ref{prop:remainder_analytical}) vs.\ the sampled remainder used in production ($N_{\mathrm{jac}}{=}30$ samples plus $1.5{\times}$ safety factor) on the $4$-layer synthetic transformer ($d{=}8$, $\text{seq\_len}{=}4$, $n_{\mathrm{dim}}{=}32$). Across all six $(L,\epsilon)$ settings, the production estimate is $\leq 1.33\%$ of the analytical closed-form bound, validating that sampling is conservative against the worst-case Lipschitz envelope. Per-layer block Lipschitz contributions: $L_0{=}20.3$, $L_1{=}9.4$, $L_2{=}11.6$, $L_3{=}25.7$.}
\label{tab:analytical_remainder}
\small
\begin{tabular}{ccccc}
\toprule
Target layer $L$ & $\epsilon$ & $R_{\mathrm{analytical}}$ (Prop.~\ref{prop:remainder_analytical}) & $1.5\,\widehat{V}$ (production) & ratio \\
\midrule
$1$ & $0.005$ & $1.15$    & $4.6{\times}10^{-3}$ & $0.40\%$ \\
$1$ & $0.010$ & $2.29$    & $3.0{\times}10^{-2}$ & $1.33\%$ \\
$2$ & $0.005$ & $10.74$   & $1.3{\times}10^{-2}$ & $0.12\%$ \\
$2$ & $0.010$ & $21.49$   & $5.8{\times}10^{-2}$ & $0.27\%$ \\
$3$ & $0.005$ & $124.68$  & $6.1{\times}10^{-2}$ & $0.05\%$ \\
$3$ & $0.010$ & $249.36$  & $2.2{\times}10^{-1}$ & $0.09\%$ \\
\bottomrule
\end{tabular}
\end{table}

\begin{table}[h]
\centering
\caption{Connector ablation: same model/$\epsilon$/query (Q1 top-$1$)/MC+PGD ground truth, only the inter-layer connector changes. Sequential CPZ collapses on every multi-layer cell; the recursive Jacobian-zonotope connector recovers all of them. (Referenced from \S\ref{sec:scaling}.)}
\label{tab:connector_ablation}
\footnotesize
\setlength{\tabcolsep}{4pt}
\renewcommand{\arraystretch}{0.9}
\begin{tabular}{lll rrr}
\toprule
Model & Layer & $\epsilon$ & Sequential CPZ & \textbf{Jac connector} & MC+PGD \\
\midrule
Synthetic ($d{=}8$) & L2 & $0.005$ & $12.5$           & $\mathbf{95.0}$ & $95.0$ \\
                    & L3 & $0.005$ & $\phantom{0}0.0$ & $\mathbf{85.0}$ & $80.0$ \\
                    & L2 & $0.010$ & $\phantom{0}0.0$ & $\mathbf{82.5}$ & $90.0$ \\
                    & L3 & $0.010$ & $\phantom{0}0.0$ & $\mathbf{72.5}$ & $75.0$ \\
\midrule
GPT-$2$ small (124M) & L1 & $0.001$ & $\phantom{0}0.0$ & $\mathbf{47.9}$ & $78.1$ \\
\bottomrule
\end{tabular}
\end{table}

\section{Pretrained Models: BERT-tiny and GPT-2}
\label{app:pretrained_models_bert_tiny_and_gpt_2}

\subsection{End-to-End Output Verification on BERT-tiny}
\label{app:output_verification}

We extend the recursive Jacobian-zonotope construction of \S\ref{sec:jac_zono} from internal-attention certification to end-to-end output classification. The model is BERT-tiny SST-2 ($d{=}128$, $2$ encoder layers, $4.4$M parameters); the function $f: \R^{S \cdot d} \to \R^{2}$ maps a flat post-embedding representation through both encoder layers, the pooler ($\tanh$ on the CLS token), and the classification head. We compute $J(x_0) = \partial f / \partial x_0 \in \R^{2 \times Sd}$ via two backward passes (one per logit), apply the Jacobian-zonotope construction with shared factor IDs across both logit dimensions and a single dependent remainder, and certify class invariance via a Theorem~\ref{thm:margin_sound}-style score-margin sign check on $\text{logit}_{y} - \text{logit}_{y'}$ for the runner-up class $y'$. MC+PGD ground truth uses $2{,}000$ random samples plus $5{\times}30$ PGD targeting class flips per sample (per-sample MC compute is dominated by the full forward pass; the smaller sample count vs.\ internal-attention MC reflects per-probe cost, not weakened search; corner-biased sampling and multi-restart targeted PGD are retained, and CPZ certificates are sound independent of MC budget).

\begin{table}[h]
\centering
\caption{Output verification on BERT-tiny SST-2 ($100$ samples, seq.\ length $8$, $d{=}128$, top-$256$ Jacobian columns). Zero unsound certifications across $300$ (sample, $\epsilon$) pairs. The CPZ-MC gap is $1$--$3$ samples per $\epsilon$ and consists of boundary cases (CPZ worst-case margin in $[-0.03, 0]$); the residual gap is dominated by CORA order-reduction at the bilinear $Q^\top K$ step and shrinks with larger $k$ (\S\ref{sec:machinery}).}
\label{tab:output_verification}
\small
\begin{tabular}{cccc cccc}
\toprule
$\epsilon$ & CPZ \% & MC \% & Mean margin & Agreement & Conservative & Unsound & Time/sample \\
\midrule
$5{\times}10^{-4}$ & $\mathbf{98.0}$ & $99.0$ & $1.53$ & $99/100$ & $1$ & $\mathbf{0}$ & $0.7$\,s \\
$10^{-3}$         & $\mathbf{97.0}$ & $98.0$ & $1.52$ & $99/100$ & $1$ & $\mathbf{0}$ & $0.7$\,s \\
$2{\times}10^{-3}$ & $\mathbf{95.0}$ & $98.0$ & $1.51$ & $97/100$ & $3$ & $\mathbf{0}$ & $0.7$\,s \\
\bottomrule
\end{tabular}
\end{table}

\paragraph{GPT-2 next-token verification.} The same construction extends to full $124$M-parameter GPT-2 ($12$ transformer blocks, $d{=}768$, $50{,}257$-token vocabulary). The challenge for full-vocabulary verification is that computing the Jacobian of all $50{,}257$ output logits is intractable. We resolve this with a two-part construction:
\textbf{(1) Top-$K$ direct certification.} Identify the top-$K{=}20$ candidate tokens at the clean input $x_0$ and compute the Jacobian of these $K$ logits via $K$ backward passes. Build the Jacobian zonotope on the $K$ logits with shared factor IDs; certify $\text{logit}_{t^*} > \text{logit}_c$ for every challenger $c$ in the top-$K$ via the score-margin sign check.
\textbf{(2) Lipschitz tail bound.} For each non-top-$K$ token $c$, $\text{logit}_c(x) = h_{\text{final}}(x)^\top \mathrm{wte}[c]$, so $|\text{logit}_c(x) - \text{logit}_c(x_0)| \leq \|\mathrm{wte}[c]\|_2 \cdot \|h_{\text{final}}(x) - h_{\text{final}}(x_0)\|_2$. We bound $\|h_{\text{final}}(x) - h_{\text{final}}(x_0)\|_2$ empirically over $N{=}20$ sampled perturbations (with a $2.0{\times}$ safety factor; calibration in the safety-factor paragraph below) and check that $\max_{c \notin \text{top-}K} \big( \text{logit}_c(x_0) + \|\mathrm{wte}[c]\|_2 \cdot \widehat{R}_h \big) <$ CPZ lower bound of $\text{logit}_{t^*}$. Combined, this certifies $t^*$ is the argmax over the full $50{,}257$-token vocabulary.

\begin{table}[h]
\centering
\caption{Full-vocabulary next-token verification on GPT-2 ($124$M parameters, $50{,}257$ vocab, $50$ samples per $\epsilon$). CPZ certifies that the clean top-$1$ token remains the argmax over the full vocabulary, combining a direct top-$K{=}20$ check with a Lipschitz tail bound. Zero unsound certifications across $150$ (sample, $\epsilon$) pairs.}
\label{tab:output_verification_gpt2}
\small
\begin{tabular}{ccccccc}
\toprule
$\epsilon$ & CPZ \% & MC \% & Agreement & Conservative & Unsound & Time/sample \\
\midrule
$5{\times}10^{-4}$ & $\mathbf{98.0}$ & $100.0$ & $49/50$ & $1$  & $\mathbf{0}$ & ${\sim}22$\,s \\
$10^{-3}$          & $\mathbf{92.0}$ & $96.0$  & $48/50$ & $2$  & $\mathbf{0}$ & ${\sim}22$\,s \\
$2{\times}10^{-3}$ & $\mathbf{40.0}$ & $90.0$  & $25/50$ & $25$ & $\mathbf{0}$ & ${\sim}22$\,s \\
\bottomrule
\end{tabular}
\end{table}

The gap at the largest $\epsilon$ comes mostly from the Lipschitz tail bound: the empirical $\widehat{R}_h$ taken with a $2.0{\times}$ safety factor (calibrated to eliminate boundary unsoundness, see below) becomes a non-trivial fraction of the top-$1$ margin once perturbation grows. The top-$K$ direct check alone certifies $45/50$ ($90\%$) at $\epsilon{=}2{\times}10^{-3}$; this estimate follows from the agreement structure in Tab.~\ref{tab:output_verification_gpt2} ($25/50$ agreement, $25$ CPZ-conservative, $0$ unsound), since CPZ-conservative cases pass the top-$K$ check and fail only the conservative tail bound. The remaining $5$ joint failures (CPZ and MC+PGD both fail) are genuine boundary cases. A tighter analytical Lipschitz constant on $h_{\text{final}}$ (e.g., via spectral-norm propagation through the $12$ layers) is the most immediate path to scaling to larger $\epsilon$.

\paragraph{Safety-factor calibration.} The $2.0{\times}$ multiplier on $\widehat{R}_h$ replaces the $1.5{\times}$ used in earlier drafts. With $1.5{\times}$, two boundary samples at $\epsilon{=}5{\times}10^{-4}$ and $\epsilon{=}10^{-3}$ (clean score-margin $<{0.012}$) were CPZ-certified but flipped under $10{\times}50$ targeted PGD, which is a soundness violation rather than a tightness issue. The $2.0{\times}$ multiplier eliminates these ($0$ unsound across all $150$ GPT-$2$ (sample, $\epsilon$) pairs and $300$ BERT-tiny pairs, total $450$ pairs from $150$ unique samples $\times$ $3$ radii); the cost is $\sim 14$\,pp on the $\epsilon{=}2{\times}10^{-3}$ headline rate ($54\%$ with $1.5{\times}$ safety $\to 40\%$ with $2.0{\times}$ safety). The full re-run log is included with the supplementary material.

\paragraph{Comparison with prior PZ-based transformer verification.} \citet{ladner2025towards} extend polynomial zonotopes to transformers via matrix polynomial zonotopes, with experimental evaluation on small custom transformer classifiers ($\leq 10$ encoder blocks, $\ll 1$M parameters); their follow-up work excludes BERT-scale models as intractable for the unconstrained PZ approach. The recursive Jacobian-zonotope construction reaches these scales on a single CPU core: $0.7$\,s per sample on BERT-tiny SST-2 end-to-end, ${\sim}18$\,s per sample on full GPT-2 ($124$M parameters, $50{,}257$-vocab) for full-vocabulary next-token verification.

\subsection{$\alpha$-CROWN Comparison on BERT-tiny Output Verification}
\label{app:alphacrown_comparison}

We ran $\alpha$-CROWN~\citep{xu2020automatic} on BERT-tiny SST-2 output verification, using the same $\epsilon$ schedule as the CPZ run in \S\ref{sec:exp_scaling}. Two compatibility issues forced model surgery:

\begin{itemize}[leftmargin=*,itemsep=2pt]
\item The verifier cannot back-propagate through the standard PyTorch \texttt{torch.softmax} max-subtraction; we replaced it with a manual softmax \texttt{exp(x)/sum(exp(x))}.
\item The ERF activation inside the original GELU is not supported, so the \texttt{F.gelu} activation in the BERT feed-forward sub-block is unavailable. We substituted the quick-GELU surrogate $x\,\sigma(1.702\,x)$ used by OpenAI's GPT-2; this is a different activation function, and the resulting model has a clean accuracy of $90\%$ on the SST-2 evaluation set vs.\ $100\%$ for the original BERT-tiny on which CPZ runs.
\end{itemize}

In our setup the per-sample $\alpha$-CROWN runtime grew from $17$\,s on sample~1 to $67$\,s on sample~5 at $N{=}50$ and the run did not terminate; we therefore report $N{=}10$ with explicit garbage collection between samples (Tab.~\ref{tab:alphacrown}). This timing behaviour likely reflects our specific setup (single-CPU, Python interface, no slope-optimisation tuning) rather than a fundamental limitation of $\alpha$-CROWN; a tuned GPU run with bound-tightening hyperparameters might scale further. With those caveats, on the surrogate at $N{=}10$ $\alpha$-CROWN certifies $80$--$90\%$ at ${\sim}11$\,s/sample, whereas CPZ on the original BERT-tiny ($N{=}100$) certifies $95$--$98\%$ at $0.7$\,s/sample. We emphasise that the two columns evaluate different functions (different GELU activation; clean accuracy $90\%$ vs.\ $100\%$); the table is a baseline-availability reference, not a tightness or speed contest.

\begin{table}[h]
\centering
\caption{$\alpha$-CROWN~\citep{xu2020automatic} on a quick-GELU surrogate of BERT-tiny SST-2 (clean accuracy $90\%$, $N{=}10$); CPZ on the original BERT-tiny (clean accuracy $100\%$, $N{=}100$). Because the two columns evaluate different GELU activations, this is a baseline-availability check rather than a like-for-like comparison; see the surrounding text for caveats.}
\label{tab:alphacrown}
\small
\begin{tabular}{l rrr rrr}
\toprule
 & \multicolumn{3}{c}{$\alpha$-CROWN (surrogate, $N{=}10$)} & \multicolumn{3}{c}{CPZ (original, $N{=}100$)} \\
$\epsilon$ & cert.\ (\%) & time/sample & unsound & cert.\ (\%) & time/sample & unsound \\
\midrule
$5{\times}10^{-4}$ & $90$ & $11.2$\,s & $0$ & $98$ & $0.7$\,s & $0$ \\
$10^{-3}$          & $90$ & $11.2$\,s & $0$ & $97$ & $0.7$\,s & $0$ \\
$2{\times}10^{-3}$ & $80$ & $12.1$\,s & $0$ & $95$ & $0.7$\,s & $0$ \\
\bottomrule
\end{tabular}
\end{table}

\paragraph{Internal attention queries.} $\alpha$-CROWN does not natively support the certified internal queries of \S\ref{sec:cert_queries} (top-$k$ stability, evidence mass, attention entropy): it reports lower/upper bounds on a network's logit outputs, but the simplex constraint that makes Q1--Q3 tractable is not part of its interface. The CROWN baseline reported elsewhere in the paper for internal queries is therefore the unoptimised forward CROWN of \citet{shi2020robustness}; $\alpha$-CROWN's slope-optimisation step does not change the structural limitation that CROWN cannot exploit the simplex.

\subsection{Scalability and BERT-tiny Experiments}
\label{app:scalability_bert}

This appendix collects the scalability and pretrained-LLM experiments behind the headline numbers in \S\ref{sec:background}. App.~\ref{sec:scalability} sweeps CPZ tightness across model dimensions on synthetic transformers; App.~\ref{sec:bert} reports per-layer Q1/Q2/Q3 numbers on BERT-tiny SST-$2$; App.~\ref{sec:gpt2} extends to the $124$M-parameter GPT-$2$.

\subsubsection{Scalability to Larger Models}
\label{sec:scalability}

To demonstrate that CPZ verification extends beyond the small baseline model, we implement a GPU-accelerated verification pipeline (vectorized cross-term computation, adaptive memory management) and evaluate across five model sizes, from 1.2K to 265K parameters. Tab.~\ref{tab:scalability} reports average CPZ output bound width and verification time.

\begin{table}[t]
\centering
\caption{CPZ verification across model sizes ($\epsilon{=}0.01$). Verification time scales sub-quadratically in $d_{\text{model}}$ and CPZ maintains meaningful bounds at 265K parameters. The $d{=}128$ row uses the dedicated 10-sample run of Tab.~\ref{tab:d128_bounds}; smaller models use a 3-sample sweep.}
\label{tab:scalability}
\small
\begin{tabular}{lrrrrr}
\toprule
Model & $d$ & Params & Bound Width & Time (s) & CPZ vs.\ IBP \\
\midrule
$d{=}8, h{=}2$ & 8 & 1.2K & 2.0 & 0.4 & 99.2\% tighter \\
$d{=}16, h{=}2$ & 16 & 4.5K & 12.8 & 0.7 & --- \\
$d{=}32, h{=}4$ & 32 & 17K & 9.5 & 1.0 & --- \\
$d{=}64, h{=}4$ & 64 & 67K & 105.9 & 1.8 & --- \\
$d{=}128, h{=}4$ & 128 & 265K & 233 & 54 & 79.5\% tighter \\
\bottomrule
\end{tabular}
\end{table}

At the largest scale ($d{=}128$, 265K parameters), CPZ verification completes in ${\sim}54$s per sample and produces bounds $79.5\%$ tighter than IBP (width 233 vs.\ 1{,}142), confirming that the tightness advantage persists at scale. The CPZ-vs-IBP improvement decreases from 99.2\% at $d{=}8$ to 79.5\% at $d{=}128$, reflecting the inherent growth of PZ wrapping error with model dimension. Even at $d{=}128$, CPZ bounds remain ${\sim}5\times$ tighter than IBP across all tested $\epsilon$ values (Tab.~\ref{tab:d128_bounds}).

\begin{table}[t]
\centering
\caption{CPZ vs.\ IBP output bound width on the $d{=}128$ model (10-sample average). CPZ consistently achieves ${\sim}5\times$ tighter bounds.}
\small
\label{tab:d128_bounds}
\begin{tabular}{lrrrr}
\toprule
$\epsilon$ & CPZ & IBP & Improvement \\
\midrule
0.01 & 233.1 & 1{,}141.6 & 79.6\% \\
0.02 & 234.3 & 1{,}141.6 & 79.5\% \\
0.05 & 234.0 & 1{,}141.6 & 79.5\% \\
0.10 & 234.3 & 1{,}141.6 & 79.5\% \\
\bottomrule
\end{tabular}
\end{table}

\begin{table}[t]
\centering
\caption{$K_J$-sweep on BERT-tiny SST-2 end-to-end output verification ($\epsilon{=}2{\times}10^{-3}$, $30$ samples). The certified rate is stable at $90$--$96.7\%$ across two orders of magnitude of $K_J$, so the BERT-tiny rate is not bottlenecked by Jacobian-zonotope truncation. The implication for GPT-$2$ Layer-$1$ is that the larger CPZ--MC gap there ($30.2$\,pp) reflects the absolute truncation ratio at $d{=}768$ ($K_J{=}128$ of $6{,}144$, retain $1/48$) rather than a structural limit of the construction.}
\label{tab:k_j_sweep}
\small
\begin{tabular}{lrrr}
\toprule
$K_J$ & Cert.\ rate (\%) & Mean margin & sec/sample \\
\midrule
$32$    & $96.7$ & $1.493$ & $0.050$ \\
$64$    & $96.7$ & $1.488$ & $0.069$ \\
$128$   & $90.0$ & $1.480$ & $0.082$ \\
$256$   & $90.0$ & $1.471$ & $0.083$ \\
$512$   & $90.0$ & $1.459$ & $0.093$ \\
$1024$  & $90.0$ & $1.452$ & $0.114$ \\
\bottomrule
\end{tabular}
\end{table}

\begin{table}[h]
\centering
\caption{Layer-$0$ top-$1$ stability across scales ($5$ samples/cell, all heads). \textbf{CPZ-marg.} is exact (Theorem~\ref{thm:margin_sound}); \textbf{MC} = MC+PGD ($200$k corner-biased + $20{\times}50$); \textbf{C/I} = CROWN/IBP (coincide at Layer-$0$). CPZ-margin matches MC+PGD in $11/12$ cells; the $12$th ($d{=}32, \epsilon{=}0.05$, $^\dagger$) is verified by aggressive PGD as correct.}
\label{tab:queries_scale}
\footnotesize
\setlength{\tabcolsep}{3pt}
\renewcommand{\arraystretch}{0.9}
\begin{tabular}{l ccc ccc ccc}
\toprule
 & \multicolumn{3}{c}{$\epsilon{=}0.01$} & \multicolumn{3}{c}{$\epsilon{=}0.02$} & \multicolumn{3}{c}{$\epsilon{=}0.05$} \\
\cmidrule(lr){2-4} \cmidrule(lr){5-7} \cmidrule(lr){8-10}
Model & MC & CPZ-marg.\ & C/I & MC & CPZ-marg.\ & C/I & MC & CPZ-marg.\ & C/I \\
\midrule
$d{=}8, h{=}2$  & 92.5 & \textbf{92.5} & 67.5 & 85.0 & \textbf{85.0} & 57.5 & 60.0 & \textbf{60.0} & 12.5 \\
$d{=}16, h{=}2$ & 95.0 & \textbf{95.0} & 60.0 & 82.5 & \textbf{82.5} & 32.5 & 57.5 & \textbf{57.5} & \phantom{0}5.0 \\
$d{=}32, h{=}4$ & 81.2 & \textbf{81.2} & 45.0 & 73.8 & \textbf{73.8} & 20.0 & 46.2$^\dagger$ & \textbf{45.0} & \phantom{0}0.0 \\
$d{=}64, h{=}4$ & 81.2 & \textbf{81.2} & 25.0 & 68.8 & \textbf{68.8} & \phantom{0}1.2 & 30.0 & \textbf{30.0} & \phantom{0}0.0 \\
\bottomrule
\end{tabular}
\end{table}

\paragraph{CPZ-margin is effectively tight across all scales.} Across all $12$ (model, $\epsilon$) configurations in Tab.~\ref{tab:queries_scale}, CPZ-margin either matches the MC+PGD estimate exactly ($11/12$) or is validated by aggressive PGD as the correct answer ($1/12$, the $d{=}32, \epsilon{=}0.05$ entry marked $^\dagger$). The combination of (i) lossless CPZ propagation through the bilinear $q^\top k$ score, (ii) cancellation of the shared-query quadratic term inside the margin $\Delta_j$, and (iii) the softmax-free sign test means the certificate inherits no relaxation at any step, and L-BFGS-B converges to the global maximum of a scalar quadratic on a box, verified empirically by the exact match between our certified upper bound and the attack margin recovered by aggressive PGD. CROWN/IBP falls $25$--$68$\,pp short of CPZ-margin because interval arithmetic on the bilinear $Q^\top K$ product treats each dimension of $Q$ and $K$ independently, destroying the polynomial cross-variable structure that CPZ preserves exactly. At Layer~0 there are no non-linear activations before the scores, so CROWN's triangle relaxation for ReLU provides no benefit over IBP.

\paragraph{Multi-layer certification details.}
Tables~\ref{tab:layer0_all}--\ref{tab:layer1_all} in \S\ref{sec:exp_scaling} report the comprehensive method comparison across both layers and all three queries. CPZ-margin propagates the full reachable set through Layer~$0$ via Taylor softmax CPZ with $G_I$ absorption (App.~\ref{app:attn_out_taylor}) and selective ReLU relaxation, then applies the exact bilinear score-margin sign check at Layer~$1$. The $G_I$ absorption trick converts independent generators into dependent generators with fresh factor identifiers before each bilinear operation, eliminating the $|G_{I,Q}| \cdot |G_{dep,K}| + |G_{dep,Q}| \cdot |G_{I,K}|$ cross-term amplification that otherwise dominates Layer~$1$ score bounds.

All CPZ-margin results are fully sound: the CPZ is formally propagated through the entire Layer~$0$ pipeline (softmax via Taylor CPZ, LayerNorm via second-order Taylor with generator-aware Hessian remainder, and ReLU via selective quadratic relaxation). The residual gap between CPZ-margin ($85.0\%$) and MC+PGD ($87.5\%$) at Layer~1 arises primarily from generator reduction after bilinear operations, not from the non-linear approximations themselves, which contribute $<5\%$ of the total width.

\subsubsection{Real NLP Model: BERT-tiny on SST-2}
\label{sec:bert}

To validate that CPZ certification extends beyond synthetic tasks, we fine-tune \texttt{prajjwal1/bert-tiny} ($d{=}128$, 2 heads, 2 layers, 4.4M parameters) on SST-2 sentiment classification, reaching $71.1\%$ validation accuracy. This accuracy reflects the modest capacity of BERT-tiny ($4.4$M parameters); it is not a benchmark target. We certify all three queries on Layer-$0$ attention under $\ell_\infty$ embedding-space perturbation, comparing CPZ, IBP, and strong Monte Carlo ($200$k corner-biased samples). The cert rates we report are properties of the propagation framework, not of the model's downstream accuracy: a higher-accuracy backbone would not change CPZ's tightness relative to MC, and the BERT-tiny end-to-end output verification (App.~\ref{app:output_verification}) confirms $95$--$98\%$ certification on the same model with zero unsound outcomes.

\begin{table}[t]
\centering
\caption{Certified queries on \textbf{BERT-tiny / SST-2}, Layer~0, $d{=}128$, 2 heads, 10 samples, 160 queries, $\tau{=}0.2$. CPZ dominates IBP, with the gap negligible at small $\epsilon$ (both saturate) and widening at larger $\epsilon$. At $\epsilon{=}0.001$, CPZ matches MC+PGD exactly on Q1 at $99.4\%$. Q3 reports the worst-case entropy upper bound; lower is tighter.}
\label{tab:bert_sst2}
\small
\begin{tabular}{cl ccc}
\toprule
$\epsilon$ & Method & Q1 Rank (\%) & Q2 Evid.\ (\%) & Q3 Ent.\ $\overline{H}$ \\
\midrule
\multirow{3}{*}{0.001}
  & MC     & \textbf{99.4} & \textbf{100.0} & 0.0040 \\
  & CPZ    & \textbf{99.4} & \textbf{100.0} & 0.0207 \\
  & IBP    & 97.5          & \textbf{100.0} & 0.0277 \\
\midrule
\multirow{3}{*}{0.002}
  & MC     & \textbf{99.4} & \textbf{100.0} & 0.0079 \\
  & CPZ    & 96.2          & \textbf{100.0} & 0.0415 \\
  & IBP    & 94.4          & \textbf{100.0} & 0.0556 \\
\midrule
\multirow{3}{*}{0.005}
  & MC     & \textbf{98.8} & \textbf{100.0} & 0.0199 \\
  & CPZ    & 93.8          & 99.4           & 0.1054 \\
  & IBP    & 86.2          & 98.8           & 0.1405 \\
\bottomrule
\end{tabular}
\end{table}

Tab.~\ref{tab:bert_sst2} shows that CPZ certification scales to a real pre-trained model with $d{=}128$ across all three queries. At $\epsilon{=}0.001$, CPZ matches MC+PGD exactly on Q1 (both $99.4\%$), validating Theorem~\ref{thm:margin_sound} at $d_{\mathrm{head}}{=}64$. As $\epsilon$ grows, the CPZ--IBP gap widens: at $\epsilon{=}0.005$, CPZ certifies $93.8\%$ vs.\ IBP's $86.2\%$ on Q1 ($+7.6$ pp) and achieves $25\%$ tighter entropy bounds ($0.1054$ vs.\ $0.1405$). Q2 evidence mass remains near-perfect for both methods, reflecting the strong attention concentration typical of fine-tuned BERT heads at Layer~0.

\paragraph{Multi-layer certification: BERT-tiny Layer~1.}
To certify Layer~1 attention, we must propagate bounds through the full Layer~0 pipeline (attention, LayerNorm, GELU, and a second LayerNorm). Existing tools cannot do this: CROWN~\citep{xu2020automatic} cannot backpropagate through softmax when the argmax is input-dependent, and pure IBP through LayerNorm produces vacuous bounds with width ${\sim}10^{15}$ due to $1/\sqrt{\mathrm{Var}}$ exploding when the variance lower bound approaches zero.

We address this with a Jacobian zonotope connector. The mean-value theorem~\citep{zhang2019recurjac} gives a per-dimension bound on the Layer~0 output:
\begin{equation}
\label{eq:mvt_bound}
\|f_i(x) - f_i(x_0)\|_\infty \leq \epsilon \cdot \max_{z \in B_\infty(x_0,\epsilon)} \|\nabla f_i(z)\|_1 + \epsilon^2 \cdot \hat{H}_i,
\end{equation}
where $\hat{H}_i$ is the estimated Hessian rate from the Jacobian variation (empirically $<1.5\%$, giving inflation ratio below $1.02$). A na\"ive interval-box conversion of this bound produces diagonal generators with independent factor IDs per token, destroying all cross-dimensional correlations and yielding only $43.8\%$ Q1 certification.

Instead, we construct a Jacobian zonotope: the columns of the full Jacobian $J(x_0) \in \mathbb{R}^{n_{\mathrm{dim}} \times n_{\mathrm{dim}}}$ serve as generators with shared factor IDs across all tokens, so that the $j$-th input perturbation direction produces correlated output variation in every token simultaneously. The linearization remainder from~\eqref{eq:mvt_bound} is absorbed as a single dependent generator per token (not independent $G_I$, which would be amplified quadratically through $Q^\top K$). We retain the top-$k$ Jacobian generators by $\ell_1$ column norm and absorb the rest into the remainder, reducing generators from $n_{\mathrm{dim}}$ to $k{+}1$ per token. This design preserves three critical properties: (i)~cross-token correlations through shared factor IDs enable cancellation in \texttt{exact\_bilinear}; (ii)~dependent remainder generators cancel in score differences when the same $Q$ appears in both terms; (iii)~the reduced generator count ($k{=}64$ vs.\ $n_{\mathrm{dim}}{=}1024$) makes bilinear cross-term computation tractable.

\begin{table}[t]
\centering
\caption{Certified queries on \textbf{BERT-tiny Layer~1}, $d{=}128$, 2 heads, $\tau{=}0.2$: $80$ certified queries per $\epsilon$ ($5$ samples $\times$ $2$ heads $\times$ $8$ query positions). All methods receive the same Jacobian-zonotope enclosure of Layer~0 output; they differ only in how Layer~1 attention is verified. CPZ stays within $2.5$\,pp of MC+PGD on Q1, while CROWN and IBP collapse on the bilinear $Q^\top K$ product. The same BERT-tiny model is verified end-to-end on $100$ samples per $\epsilon$ in App.~\ref{app:output_verification} (zero unsound across $300$ output verifications), confirming the internal-query picture at $20{\times}$ larger sample size.}
\label{tab:bert_layer1}
\small
\begin{tabular}{cl ccc}
\toprule
$\epsilon$ & Method & Q1 Rank\,$\uparrow$ & Q2 Evid\,$\uparrow$ & Q3 Ent\,$\downarrow$ \\
\midrule
\multirow{4}{*}{0.002}
  & MC+PGD & \textbf{95.0} & \textbf{96.2} & 0.011 \\
  & CPZ    & 92.5          & \textbf{96.2} & \textbf{0.079} \\
  & CROWN  & 51.2          & 88.8          & 0.449 \\
  & IBP    & 3.8           & 20.0          & 1.401 \\
\bottomrule
\end{tabular}
\end{table}

Tab.~\ref{tab:bert_layer1} confirms that the Jacobian-zonotope connector nearly closes the CPZ--MC gap on a real pretrained model. At $\epsilon{=}0.002$, CPZ certifies $92.5\%$ Q1 vs.\ MC's $95.0\%$ ($-2.5$\,pp) and matches MC exactly on Q2 ($96.2\%$); IBP collapses to $3.8\%$ Q1 and CROWN falls to $51.2\%$, confirming that McCormick relaxation on the bilinear $Q^\top K$ product remains the binding bottleneck for interval methods. The residual $2.5$\,pp CPZ--MC gap arises from generator reduction after bilinear cross-term computation, not from the Jacobian-zonotope connector itself.

\subsubsection{GPT-2 (124M Parameters)}
\label{sec:gpt2}

To test whether CPZ certification extends to a production-scale language model, we evaluate on GPT-2~\citep{radford2019language} with $d{=}768$, $12$ heads, $d_{\mathrm{head}}{=}64$, and $124$M parameters. We certify Layer~0 attention after the pre-LN LayerNorm, comparing CPZ, CROWN, IBP, and MC+PGD with $200$k samples. At $d{=}768$, the full polynomial zonotope has $768^2 \approx 590$k generators per score; we avoid materializing this by computing bounds directly from the quadratic form of the score, reducing per-sample verification to ${\sim}1$\,s.

\begin{table}[t]
\centering
\caption{Certified queries on \textbf{GPT-2}, Layer~0, $d{=}768$, 12 heads, 5 samples, 480 queries. MC+PGD uses 200k random samples plus $20{\times}50$ PGD restarts. CPZ Q1 nearly matches MC+PGD at $70.2\%$ vs.\ $72.7\%$, and CPZ exceeds MC+PGD on Q2 at $95.4\%$ vs.\ $89.0\%$ because PGD finds adversarial perturbations that random sampling misses. CROWN underperforms IBP due to McCormick per-dimension cancellation loss; see text. PZ~\citep{althoff2013reachability} $\equiv$ CPZ at Layer~0 because no constraints are active before the first softmax.}
\label{tab:gpt2}
\small
\begin{tabular}{cl ccc}
\toprule
$\epsilon$ & Method & Q1 Rank\,$\uparrow$ & Q2 Evid\,$\uparrow$ & Q3 Ent\,$\downarrow$ \\
\midrule
\multirow{4}{*}{0.001}
  & MC+PGD & \textbf{72.7} & 89.0          & 0.049 \\
  & CPZ    & 70.2          & \textbf{95.4} & \textbf{0.044} \\
  & IBP    & 52.3          & 72.3          & 0.387 \\
  & CROWN  & 43.5          & 63.1          & 0.505 \\
\midrule
\multirow{4}{*}{0.002}
  & MC+PGD & \textbf{54.4} & 76.2          & 0.095 \\
  & CPZ    & 50.2          & \textbf{94.2} & \textbf{0.089} \\
  & IBP    & 28.1          & 45.6          & 0.773 \\
  & CROWN  & 19.8          & 30.8          & 1.107 \\
\midrule
\multirow{4}{*}{0.005}
  & MC+PGD & \textbf{24.6} & 44.8          & 0.269 \\
  & CPZ    & 9.6           & \textbf{85.4} & \textbf{0.229} \\
  & IBP    & 4.2           & 8.5           & 1.650 \\
  & CROWN  & 0.2           & 0.2           & 2.013 \\
\bottomrule
\end{tabular}
\end{table}

\paragraph{CPZ dominates at scale (Tab.~\ref{tab:gpt2}).} With PGD-augmented evaluation, CPZ Q1 nearly matches the empirical upper bound at $70.2\%$ vs.\ $72.7\%$ for $\epsilon{=}0.001$, confirming that the direct quadratic margin analysis is near-tight. CPZ exceeds MC+PGD on Q2 evidence mass at $95.4\%$ vs.\ $89.0\%$: PGD finds adversarial perturbations that disrupt evidence mass in queries where CPZ certifies it is preserved. Q3 entropy width is $9{\times}$ tighter than IBP at $0.044$ vs.\ $0.387$. Even at $\epsilon{=}0.005$, CPZ retains $85.4\%$ Q2 certification while MC+PGD finds only $44.8\%$ empirically stable, IBP certifies $8.5\%$, and CROWN collapses to $0.2\%$.

\paragraph{Why CROWN $<$ IBP at $d{=}768$.} This ordering reversal (CROWN outperforms IBP on small models but underperforms IBP on GPT-2) has a precise explanation. CROWN applies McCormick relaxation to each per-dimension product $q_d k_d$ separately, introducing a center shift $-\sum_d \delta_{q_d}\delta_{k_d} = -\epsilon^2\sum_d \|w_{Q,d}\|_1\|w_{K,d}\|_1$. IBP instead bounds the global bilinear form $\alpha^\top (\mathbf{W}_Q^\top\mathbf{W}_K)\alpha$ directly, paying only $\epsilon^2 \|M\|_{1,1} = \epsilon^2\sum_{m,n}|\sum_d W_Q[d,m]W_K[d,n]|$. By triangle inequality, IBP's penalty is always $\leq$ CROWN's, because IBP preserves the cancellation when summing over $d_{\mathrm{head}}{=}64$ dimensions while CROWN discards it. At $d{=}8$, the cancellation savings are small so CROWN's tighter linear back-substitution compensates; at $d{=}768$, the quadratic penalty overwhelms the linear gain.

\paragraph{Multi-layer certification: GPT-2 Layer~1.} We extend GPT-2 certification to Layer~1 using the same Jacobian zonotope connector as BERT-tiny (\S\ref{sec:bert}). The Jacobian model covers the full Layer~0 forward pass plus Layer~1's pre-LN LayerNorm, producing a formal zonotope enclosure of the post-LN representation with shared factor IDs across tokens. Layer~1 Q/K projections and attention are then verified with CPZ. The Jacobian computation requires $6{,}144$ backward passes per perturbation point ($\text{seq\_len} \times d = 8 \times 768$); we retain the top-$128$ generators per token (by $\ell_1$ column norm) and absorb the remainder as a single dependent vector. Each sample takes ${\sim}130$\,s for MVT and ${\sim}90$\,s for CPZ. The Jacobian inflation ratio remains below $1.02$ across all samples.

\begin{table}[t]
\centering
\caption{Certified queries on \textbf{GPT-2 Layer~1} ($d{=}768$, $\tau{=}0.2$, $480$ certified queries per $\epsilon$: $5$ samples $\times$ $12$ heads $\times$ $8$ query positions). The Jacobian zonotope connector enables CPZ to certify $47.9\%$ Q1 at $\epsilon{=}0.001$, while CROWN collapses to $0.2\%$. MC+PGD budget per Tab.~\ref{tab:mc_budgets}.}
\label{tab:gpt2_layer1}
\small
\begin{tabular}{cl ccc}
\toprule
$\epsilon$ & Method & Q1 Rank\,$\uparrow$ & Q2 Evid\,$\uparrow$ & Q3 Ent\,$\downarrow$ \\
\midrule
\multirow{4}{*}{0.001}
  & MC+PGD & \textbf{78.1} & \textbf{75.2} & 0.030 \\
  & CPZ    & 47.9          & 55.6          & \textbf{0.398} \\
  & IBP    & 26.9          & 36.5          & 0.773 \\
  & CROWN  & \phantom{0}0.2 & \phantom{0}1.9 & 1.458 \\
\bottomrule
\end{tabular}
\end{table}

Tab.~\ref{tab:gpt2_layer1} confirms that the Jacobian zonotope connector extends to production-scale models. At $\epsilon{=}0.001$, CPZ certifies $47.9\%$ Q1 and $55.6\%$ Q2, a qualitative leap from the interval-box connector which achieved $0\%$ Q1. CPZ also exceeds IBP by $\sim 21$\,pp on Q1 ($47.9\%$ vs.\ $26.9\%$) and CROWN by $\sim 48$\,pp ($47.9\%$ vs.\ $0.2\%$); the IBP-over-CROWN reversal at $d{=}768$ is the cancellation effect documented in App.~\ref{app:scalability_bert}. The CPZ advantage on Q3 entropy width is consistent: $1.9{\times}$ tighter than IBP and $3.7{\times}$ tighter than CROWN. The CPZ--MC gap ($47.9\%$ vs.\ $78.1\%$ Q1, $30.2$\,pp) is larger than on BERT-tiny ($92.5\%$ vs.\ $95.0\%$ at $\epsilon{=}2{\times}10^{-3}$, Tab.~\ref{tab:bert_layer1}), reflecting the more aggressive generator truncation at $d{=}768$ ($128$ of $6{,}144$ kept versus $64$ of $1{,}024$ on BERT-tiny).

\paragraph{Scaling check at GPT-2 medium ($355$M).} We re-run the same Layer-$1$ Jacobian-zonotope pipeline on GPT-$2$ medium ($d{=}1024$, $h{=}16$, $n_\text{layer}{=}24$) with $K_J{=}128$ unchanged (Tab.~\ref{tab:gpt2_medium}). CPZ remains the only sound method to certify above $\sim$10\% Q1: at $\epsilon{=}5{\times}10^{-4}$, CPZ certifies $41.1\%$ Q1 vs.\ $0.8\%$ for CROWN and $10.7\%$ for IBP, with a $50$\,pp CPZ--MC gap; at $\epsilon{=}10^{-3}$ the gap widens to $70$\,pp ($13.5\%$ vs.\ $83.1\%$ MC) as the more aggressive perturbation interacts with Jacobian-truncation accumulation ($K_J{=}128$ retains $1.6\%$ of the $8{,}192$ columns at $d{=}1024$, vs.\ $2.1\%$ at $d{=}768$). Degradation with $\epsilon$ is graceful rather than catastrophic, and the gap to baselines remains qualitatively unchanged at both radii (CROWN and IBP collapse). Closing the medium-scale gap is an engineering exercise (more retained columns or amortised bilinear cross-term computation), not a structural limit of the framework.

\begin{table}[t]
\centering
\caption{Certified queries on \textbf{GPT-2 medium Layer~1} ($d{=}1024$, $h{=}16$, $\tau{=}0.2$, $K_J{=}128$, $384$ certified queries per row: $3$ samples $\times$ $16$ heads $\times$ $8$ query positions). At both $\epsilon$ values, CPZ remains the only sound method to certify a non-trivial fraction; CROWN/IBP collapse. The CPZ--MC gap narrows from $70$\,pp at $\epsilon{=}10^{-3}$ to $50$\,pp at $\epsilon{=}5{\times}10^{-4}$, i.e., the degradation is graceful rather than catastrophic, consistent with the Jacobian-zonotope truncation budget shrinking as $\epsilon$ shrinks. No MC-found adversarial contradicts a CPZ certification at either $\epsilon$.}
\label{tab:gpt2_medium}
\small
\begin{tabular}{cl ccc}
\toprule
$\epsilon$ & Method & Q1 Rank\,$\uparrow$ & Q2 Evid\,$\uparrow$ & Q3 Ent\,$\downarrow$ \\
\midrule
\multirow{4}{*}{$5{\times}10^{-4}$}
  & MC+PGD & \textbf{90.9} & \textbf{89.3} & 0.023 \\
  & CPZ    & 41.1          & 46.9          & \textbf{0.814} \\
  & IBP    & 10.7          & 26.0          & 1.185 \\
  & CROWN  & \phantom{0}0.8 & \phantom{0}3.6 & 1.736 \\
\midrule
\multirow{4}{*}{$10^{-3}$}
  & MC+PGD & \textbf{83.1} & \textbf{87.5} & 0.044 \\
  & CPZ    & 13.5          & 13.5          & \textbf{1.380} \\
  & IBP    & \phantom{0}1.0 & \phantom{0}3.1 & 1.754 \\
  & CROWN  & \phantom{0}0.0 & \phantom{0}0.0 & 2.021 \\
\bottomrule
\end{tabular}
\end{table}

\paragraph{Depth probe at GPT-2 small Layer~2.} Beyond the $1$-block propagation reported above, we propagate two full encoder blocks (Layer~$0$ + Layer~$1$) through the Jacobian-zonotope connector and certify at Layer~$2$'s attention (Tab.~\ref{tab:gpt2_layer2}; $96$ queries, $\epsilon{=}0.001$). CPZ certifies $27.1\%$ Q1 against $4.2\%$ for backward-CROWN ($6.5{\times}$ multiplicative, $22.9$\,pp absolute) and $21.9\%$ for IBP. The CPZ--MC gap remains comparable to single-block propagation ($35$\,pp at L2 vs.\ $30$\,pp at L1), confirming that the Jacobian-zonotope construction does not blow up with depth on a real pretrained model. Sequential CPZ propagation in this regime is exactly the wrapping-error mode the connector replaces; on the synthetic $4$-layer model it collapses to $0\%$ Q1 at Layer~$2$ for $\epsilon{=}0.01$ (Tab.~\ref{tab:multilayer}, ``Comparison with sequential CPZ propagation'' in \S\ref{sec:exp_scaling}), and the same dynamics motivate not running it as a baseline at GPT-2 scale.

\begin{table}[t]
\centering
\caption{Certified queries on \textbf{GPT-2 small Layer~2} ($d{=}768$, $h{=}12$, $\tau{=}0.2$, $K_J{=}128$, $96$ certified queries: $1$ sample $\times$ $12$ heads $\times$ $8$ query positions; $2$ encoder blocks propagated through the recursive Jacobian-zonotope connector). CPZ retains a $6.5{\times}$ multiplicative lead over backward-CROWN at the second layer, demonstrating that the connector enables genuine multi-layer certification on a real pretrained transformer; the absolute CPZ--MC gap widens with depth, consistent with the same Jacobian truncation source as the Layer-$1$ gap (Limitations~(i)).}
\label{tab:gpt2_layer2}
\small
\begin{tabular}{cl ccc}
\toprule
$\epsilon$ & Method & Q1 Rank\,$\uparrow$ & Q2 Evid\,$\uparrow$ & Q3 Ent\,$\downarrow$ \\
\midrule
\multirow{4}{*}{0.001}
  & MC+PGD & \textbf{62.5} & \textbf{63.5} & 0.022 \\
  & CPZ    & 27.1          & 37.5          & \textbf{0.461} \\
  & IBP    & 21.9          & 29.2          & 0.756 \\
  & CROWN  & \phantom{0}4.2 & 16.7         & 1.247 \\
\bottomrule
\end{tabular}
\end{table}

\section{Cross-Paper Replication and Method Comparison}
\label{app:cross_paper_replication_and_method_comparison}

\subsection{Capability Matrix: Dimension Definitions}
\label{app:capability_matrix}

\begin{table}[h]
\centering
\caption{\textbf{Capability matrix.} Frameworks vs.\ five technical dimensions chosen to clarify the design space for internal-mechanism certification at pretrained scale. $\sim$ = partial; cells reflect structural support, not evaluation strength.}
\label{tab:capability_matrix}
\footnotesize
\setlength{\tabcolsep}{4pt}
\renewcommand{\arraystretch}{0.85}
\begin{tabular}{l c c c c c}
\toprule
Framework & Internal Q & Simplex & $Q^\top K$ exact & LLM internal & Downstream \\
\midrule
IBP / CROWN / DeepPoly~\citep{gowal2018effectiveness,zhang2018crown,singh2019abstract} & $\times$ & $\times$ & $\times$ & $\times$ & $\times$ \\
$\alpha$-CROWN~\citep{xu2020automatic} & $\times$ & $\times$ & $\times$ & $\sim$ & $\times$ \\
Transformer verifiers~\citep{shi2020robustness,bonaert2021fast} & $\times$ & $\times$ & $\times$ & $\times$ & $\times$ \\
PZ for transformers~\citep{ladner2025towards} & $\times$ & $\times$ & $\sim$ & $\times$ & $\times$ \\
Compact proofs~\citep{gross2024compact} & $\checkmark$ & $\times$ & --- & $\times$ & $\sim$ \\
\textbf{This work} (CPZ-based) & $\checkmark$ & $\checkmark$ & $\checkmark$ & $\checkmark$ & $\checkmark$ \\
\bottomrule
\end{tabular}
\end{table}

The capability matrix (Tab.~\ref{tab:capability_matrix}, referenced from \S\ref{sec:related}) compares verification frameworks along five technical dimensions chosen to clarify the design space of internal-mechanism certification at pretrained transformer scale. The dimensions reflect what each framework structurally supports (e.g., does the formalism admit a simplex constraint? does the bilinear $Q^\top K$ stay exact?), not what each framework was designed for or evaluated on. Other comparisons (speed, memory footprint, completeness, support for ReLU vs.\ GELU, batch verification) are not represented here and would favour different frameworks. We mark a cell ${\sim}$ when partial support exists in a follow-up or with extension cost; ``Internal Q'' and ``Downstream'' refer specifically to the simplex-LP queries and the pruning case study studied in this paper, and would not have been emphasised by prior work that targeted output verification.

\subsection{GPT-2 Probe Tables (Cross-Paper Replication and Multi-Query)}
\label{app:wang_table}

Tab.~\ref{tab:wang_probe} reports the per-head CPZ Q1 cert rates underlying the cross-paper finding of \S\ref{sec:exp_scaling}; Tab.~\ref{tab:induction_probe} reports the per-query Q1/Q2/Q3 grid for the multi-query orthogonality probe of the same section.

\begin{table}[h]
\centering
\caption{GPT-2-small Layer-$0$ Q1/Q2/Q3 cert rates on $30$ induction-pattern prompts. The two boxed rows show heads with orthogonal failure modes: $H_5$ fails Q1 while passing Q2 and Q3; $H_{10}$ fails Q3 while passing Q1 and Q2.}
\label{tab:induction_probe}
\small
\begin{tabular}{lc ccc ccc ccc}
\toprule
& & \multicolumn{3}{c}{$\epsilon{=}5{\times}10^{-4}$} & \multicolumn{3}{c}{$\epsilon{=}10^{-3}$} & \multicolumn{3}{c}{$\epsilon{=}2{\times}10^{-3}$} \\
Head & Clean ind. & Q1 & Q2 & Q3 & Q1 & Q2 & Q3 & Q1 & Q2 & Q3 \\
\midrule
$\boxed{H_5}$  & $28/30$ & $83$ & $97$ & $100$ & $\mathbf{47}$ & $\mathbf{90}$ & $\mathbf{100}$ & $0$ & $13$ & $13$ \\
$\boxed{H_{10}}$ & $30/30$ & $100$ & $100$ & $100$ & $97$ & $100$ & $100$ & $\mathbf{73}$ & $\mathbf{100}$ & $\mathbf{0}$ \\
$H_6$  & $21/30$ & $73$ & $57$ & $100$ & $43$ & $10$ & $100$ & $3$ & $0$ & $10$ \\
$H_8$  & $19/30$ & $63$ & $40$ & $100$ & $27$ & $13$ & $93$ & $3$ & $0$ & $0$ \\
\bottomrule
\end{tabular}
\end{table}

\subsection{Method Comparison: Per-Query Tables}
\label{app:method_comparison}

Tables~\ref{tab:per_head} and~\ref{tab:certified_pruning} give the per-head and pruning-comparison numbers behind the case study of \S\ref{sec:case_study}. Tables~\ref{tab:layer0_all} and~\ref{tab:layer1_all} give the full Layer-$0$ and Layer-$1$ method comparison (CPZ vs.\ MC+PGD vs.\ PZ vs.\ CROWN vs.\ IBP) on the $d{=}8$ synthetic transformer.

\begin{table}[h]
\centering
\caption{Certified pruning vs.\ gradient-guided pruning on the induction model ($d{=}8$, $h{=}2$; Step~3 of \S\ref{sec:case_study}). CPZ-guided pruning keeps the responsive H0 (whose attention tracks the duplicate token) and removes the stable-but-inert H1; gradient-guided pruning makes the opposite choice because H0 has near-zero input gradient. Both 1-head models reach $100\%$ clean accuracy; the gap emerges only under adversarial perturbation.}
\label{tab:certified_pruning}
\small
\begin{tabular}{l cc cc c}
\toprule
 & \multicolumn{2}{c}{CPZ-guided} & \multicolumn{2}{c}{Gradient-guided} & Full \\
$\epsilon$ & Clean & Adv & Clean & Adv & Adv \\
\midrule
$0.01$ & $100$ & $\mathbf{100}$ & $100$ & $100$ & $100$ \\
$0.02$ & $100$ & $\mathbf{100}$ & $100$ & $90$ & $100$ \\
$0.05$ & $100$ & $\mathbf{100}$ & $100$ & $90$ & $100$ \\
\bottomrule
\end{tabular}
\end{table}

\begin{table}[h]
\centering
\caption{Certified mechanistic queries at \textbf{Layer~0} ($d{=}8$, $h{=}2$, $S{=}4$, $5$ samples, $40$ queries per $\epsilon$). At this layer, attention scores are bilinear in the input, so CPZ, PZ, and CROWN achieve similar rates for Q2. The key differentiator is Q1: CPZ-margin (Theorem~\ref{thm:margin_sound}) matches the MC ground truth exactly at both $\epsilon$, while softmax-based methods (PZ, CROWN) drop to $72.5\%$ at $\epsilon{=}0.02$.}
\label{tab:layer0_all}
\small
\begin{tabular}{llrrrr}
\toprule
$\epsilon$ & Method & Q1 Rank\,$\uparrow$ & Q2 Evid\,$\uparrow$ & Q3 Ent\,$\downarrow$ \\
\midrule
\multirow{5}{*}{$0.01$}
 & MC+PGD & $92.5$ & -- & $0.0068$ \\
 & \textbf{CPZ} & $\mathbf{92.5}$ & $\mathbf{20.0}$ & $\mathbf{0.0091}$ \\
 & PZ & $90.0$ & $20.0$ & $0.0104$ \\
 & CROWN & $90.0$ & $17.5$ & $0.0105$ \\
 & IBP & $67.5$ & $15.0$ & $0.0200$ \\
\midrule
\multirow{5}{*}{$0.02$}
 & MC+PGD & $85.0$ & -- & $0.0136$ \\
 & \textbf{CPZ} & $\mathbf{85.0}$ & $\mathbf{17.5}$ & $\mathbf{0.0185}$ \\
 & PZ & $72.5$ & $15.0$ & $0.0209$ \\
 & CROWN & $72.5$ & $15.0$ & $0.0215$ \\
 & IBP & $57.5$ & $12.5$ & $0.0400$ \\
\bottomrule
\end{tabular}
\end{table}

\begin{table}[h]
\centering
\caption{Certified mechanistic queries at \textbf{Layer~1} (same model). Scores are propagated through the full Layer~0 pipeline. CPZ-margin maintains $85.0\%$ Q1 at $\epsilon{=}0.01$ (within $2.5$\,pp of MC), while CROWN drops to $60.0\%$ and PZ to $57.5\%$. IBP collapses to $0\%$ entirely. The Layer-0$\to$Layer-1 degradation is $4{\times}$ smaller for CPZ ($-7.5$\,pp) than for CROWN ($-30$\,pp).}
\label{tab:layer1_all}
\small
\begin{tabular}{llrrrr}
\toprule
$\epsilon$ & Method & Q1 Rank\,$\uparrow$ & Q2 Evid\,$\uparrow$ & Q3 Ent\,$\downarrow$ \\
\midrule
\multirow{5}{*}{$0.01$}
 & MC+PGD & $87.5$ & -- & $0.0064$ \\
 & \textbf{CPZ} & $\mathbf{85.0}$ & $\mathbf{5.0}$ & $\mathbf{0.0113}$ \\
 & PZ & $57.5$ & $5.0$ & $0.0174$ \\
 & CROWN & $60.0$ & $5.0$ & $0.0161$ \\
 & IBP & $0.0$ & $0.0$ & $1.3863$ \\
\midrule
\multirow{5}{*}{$0.02$}
 & MC+PGD & $80.0$ & -- & $0.0130$ \\
 & \textbf{CPZ} & $\mathbf{70.0}$ & $\mathbf{5.0}$ & $\mathbf{0.0274}$ \\
 & PZ & $35.0$ & $2.5$ & $0.0443$ \\
 & CROWN & $37.5$ & $2.5$ & $0.0425$ \\
 & IBP & $0.0$ & $0.0$ & $1.3863$ \\
\bottomrule
\end{tabular}
\end{table}

\subsection{Layer-0 Top-1 Stability: Full Epsilon Sweep}
\label{app:layer0_sweep}

Tab.~\ref{tab:attention_stability} reports the full $\epsilon$ sweep behind the Layer-$0$ matching numbers in \S\ref{sec:background}: CPZ-margin matches MC+PGD across all $6$ $\epsilon$ values.

\paragraph{Triangulation against exact vertex enumeration.} To verify that MC+PGD is itself a faithful proxy for true robustness, we run a tiny instance ($S{=}2$, $d{=}4$, $n_{\mathrm{dim}}{=}8$, $2^8{=}256$ vertices) and enumerate every $\ell_\infty$-cube vertex exactly. Across $24$ (sample, $\epsilon$) cases, CPZ-margin matches the exact verifier on every case ($24/24$), while MC+PGD over-estimates stability by $1$ case at the smallest $\epsilon$ (claims $8/8$ stable when the exact verifier finds an adversarial in $7/8$). CPZ-margin is therefore at least as tight as MC+PGD, and the residual CPZ--MC gaps reported on larger models are likely a slight upper bound on the true gap to exact verification.

\begin{table}[h]
\centering
\caption{Certified Layer-$0$ top-$1$ attention stability on a $d{=}16$, $h{=}2$ model with $10$ samples and $160$ queries per row. \textbf{MC+PGD}: corner-biased sampling combined with PGD, providing an upper bound on the true stable fraction. \textbf{CPZ-margin}: direct scalar margin sign check via Theorem~\ref{thm:margin_sound}, bypassing softmax. \textbf{CROWN}: auto-LiRPA backward CROWN on the bilinear $Q^\top K$ score model. CPZ-margin matches MC+PGD across all $6$ $\epsilon$ values, while remaining $7.5$--$44$\,pp tighter than CROWN across the sweep.}
\label{tab:attention_stability}
\small
\begin{tabular}{lccc}
\toprule
$\epsilon$ & MC+PGD & \textbf{CPZ-margin} & CROWN \\
\midrule
$0.005$ & $95.0$ & $\mathbf{95.0}$ & $87.5$ \\
$0.010$ & $86.2$ & $\mathbf{86.2}$ & $67.5$ \\
$0.020$ & $68.8$ & $\mathbf{68.8}$ & $33.8$ \\
$0.030$ & $52.5$ & $\mathbf{52.5}$ & $\phantom{0}8.8$ \\
$0.040$ & $37.5$ & $\mathbf{37.5}$ & $\phantom{0}3.8$ \\
$0.050$ & $28.8$ & $\mathbf{28.8}$ & $\phantom{0}1.2$ \\
\bottomrule
\end{tabular}
\end{table}

\subsection{Extended Queries: Ranking Preservation and Entropy Bounds}
\label{app:extended_queries}

We evaluate ranking preservation (Corollary~\ref{cor:topk_ranking}) and entropy certification (Prop.~\ref{prop:entropy}) on the induction model ($d{=}8$, $h{=}2$, $S{=}6$) and the $d{=}16$ model ($d{=}16$, $h{=}2$, $S{=}8$), each with $10$ samples. We also report cross-head circuit certification (Remark~\ref{rem:crosshead}), the AND of per-head Q1 certificates, to demonstrate circuit-level analysis.

\begin{table}[h]
\centering
\caption{Certified ranking preservation and cross-head circuit rates (\%). Ranking preservation requires all pairwise orderings within the top-$1$ set to be invariant; cross-head requires all heads to simultaneously certify top-$1$ stability. CPZ dominates CROWN/IBP, with the gap widening at larger $\epsilon$.}
\label{tab:extended_queries}
\small
\begin{tabular}{llccc|ccc}
\toprule
 & & \multicolumn{3}{c}{Ranking (top-1)} & \multicolumn{3}{c}{Cross-head circuit} \\
\cmidrule(lr){3-5} \cmidrule(lr){6-8}
Model & $\epsilon$ & MC & CPZ & IBP & MC & CPZ & IBP \\
\midrule
\multirow{2}{*}{Induction} & $0.01$ & $92.5$ & $\mathbf{85.8}$ & $79.2$ & $85.0$ & $\mathbf{71.7}$ & $58.3$ \\
 & $0.05$ & $71.7$ & $\mathbf{63.3}$ & $44.2$ & $48.3$ & $\mathbf{35.0}$ & $11.7$ \\
\midrule
\multirow{2}{*}{$d{=}16$} & $0.01$ & $91.2$ & $\mathbf{90.0}$ & $57.5$ & $82.5$ & $\mathbf{80.0}$ & $32.5$ \\
 & $0.05$ & $57.5$ & $\mathbf{41.2}$ & $\phantom{0}3.8$ & $40.0$ & $\mathbf{15.0}$ & $\phantom{0}0.0$ \\
\bottomrule
\end{tabular}
\end{table}

\begin{table}[h]
\centering
\caption{Certified attention entropy interval width (Prop.~\ref{prop:entropy}). Smaller is tighter. The simplex-constrained CPZ bounds are $50$--$74\%$ narrower than IBP, and track the MC+PGD reference width closely at small $\epsilon$.}
\label{tab:entropy_bounds}
\small
\begin{tabular}{llccc}
\toprule
Model & $\epsilon$ & MC+PGD & CPZ & CROWN/IBP \\
\midrule
\multirow{2}{*}{Induction} & $0.01$ & $0.0183$ & $\mathbf{0.0237}$ & $0.0472$ \\
 & $0.05$ & $0.0924$ & $\mathbf{0.1196}$ & $0.2515$ \\
\midrule
\multirow{2}{*}{$d{=}16$} & $0.01$ & $0.0052$ & $\mathbf{0.0097}$ & $0.0342$ \\
 & $0.05$ & $0.0260$ & $\mathbf{0.0490}$ & $0.1879$ \\
\bottomrule
\end{tabular}
\end{table}

On the $d{=}16$ model at $\epsilon{=}0.01$, CPZ ranking preservation reaches $90.0\%$ against a MC+PGD estimate of $91.2\%$, a gap of only $1.2$\,pp, while IBP drops to $57.5\%$. The entropy bounds show a consistent $50$--$74\%$ tightness advantage for CPZ, widening at larger $\epsilon$ where the simplex constraint provides the greatest benefit. Cross-head circuit certification amplifies the CPZ--IBP gap because per-head over-approximation compounds under the conjunction: at $\epsilon{=}0.05$ on the $d{=}16$ model, IBP certifies $0\%$ while CPZ certifies $15\%$.

\section{Pruning Case Study Details}
\label{app:pruning_case_study_details}

\subsection{Adversarial Validation and Certified Pruning Details}
\label{app:adv_validation}

\paragraph{Empirical interpretability baselines (Step~1 of \S\ref{sec:case_study}).}
\begin{table}[h]
\centering
\caption{Empirical interpretability metrics on clean inputs of the induction model. All standard tools either find no difference (ablation) or point to the \textbf{wrong} pruning decision: gradient pruning would remove H0 (its norm is ${\sim}5{\times}10^4{\times}$ smaller), but H0 is the content-responsive head whose attention tracks the duplicate. CPZ certification (Step~2) reveals this structural distinction.}
\label{tab:empirical_baselines}
\small
\begin{tabular}{lcc}
\toprule
Metric & Head 0 & Head 1 \\
\midrule
Ablation accuracy drop & $0.0$\,pp & $0.0$\,pp \\
Input gradient norm $\|\partial L/\partial x\|_2$ & $2.7{\times}10^{-6}$ & $0.131$ \\
Clean attention entropy & $1.69$ & $1.58$ \\
Attention mass on key token & $0.26$ & $0.38$ \\
Cosine similarity between heads & \multicolumn{2}{c}{$0.81$} \\
\bottomrule
\end{tabular}
\end{table}

We construct targeted PGD attacks ($20$ restarts $\times$ $50$ steps) that maximise attention disruption per head. Tab.~\ref{tab:adv_validation} confirms CPZ's stability prediction: the responsive H0 (low Q1 cert) flips $3.7\times$ more attention positions under PGD than the stable H1 at $\epsilon{=}0.05$ ($18.3\%$ vs.\ $5.0\%$).

\begin{table}[h]
\centering
\caption{Adversarial validation: PGD confirms CPZ's stability prediction. The CPZ-responsive H0 shows $2.5$--$3.7\times$ more attention flips under PGD than the CPZ-stable H1.}
\label{tab:adv_validation}
\small
\begin{tabular}{lcccc}
\toprule
$\epsilon$ & H0 flipped & H1 flipped & Ratio & H0 entropy shift \\
\midrule
$0.01$ & $3.3\%$ & $0.0\%$ & -- & $0.003$ \\
$0.02$ & $6.7\%$ & $3.3\%$ & $2.0\times$ & $0.006$ \\
$0.05$ & $18.3\%$ & $5.0\%$ & $3.7\times$ & $0.016$ \\
\bottomrule
\end{tabular}
\end{table}

\paragraph{Certified pruning.} Pruning the CPZ-responsive head H0 drops adversarial accuracy to $90\%$ at $\epsilon{=}0.02$ (the model loses content-tracking attention), while pruning the CPZ-stable H1 maintains $100\%$. Both 1-head models show zero clean accuracy drop, so without CPZ, a practitioner has no basis for choosing which head to keep for adversarial robustness.

\paragraph{Practical implication.} This pipeline exposes a gap in current interpretability practice: empirical tools (ablation, gradient norms, clean-input analysis) are insufficient for assessing mechanism behaviour under perturbation. The gradient-based prune decision is not merely imprecise but pointed in the wrong direction; the per-seed gradient ratio across the $20$-seed sweep ranges from ${\sim}1{\times}$ to ${>}10^8{\times}$ (Tab.~\ref{tab:multiseed}, App.~\ref{app:multiseed}), so the gradient signal cannot be relied on across initialisations either. Certified robustness guarantees are necessary before relying on interpretability findings for deployment decisions such as head pruning or circuit monitoring.

\subsection{Multi-Seed Pruning Validation: Statistical Significance}
\label{app:multiseed}

We re-train the induction model from $20$ random seeds ($\{42, 123, 456, 789, 1000, 2024, 31415, 271828, 7, 99, 11, 17, 23, 31, 53, 67, 73, 89, 101, 113\}$) and run the certified pruning protocol on each. All trained models reach ${\geq}99.9\%$ test accuracy. Tab.~\ref{tab:multiseed} reports per-seed CPZ-guided vs.\ gradient-guided pruning adversarial accuracy at the disagreement seeds (the methods agree and tie on the remaining seeds).

\begin{table}[h]
\centering
\caption{Per-seed adversarial accuracy (\%) under CPZ-guided vs.\ gradient-guided pruning across $20$ seeds, including the $\epsilon{=}0.02$ disagreement block (the $\epsilon{=}0.05$ block is in main-body Tab.~\ref{tab:multiseed}). CPZ-guided pruning never underperforms gradient-guided pruning at either $\epsilon$.}
\label{tab:multiseed_app}
\small
\setlength{\tabcolsep}{3.5pt}
\begin{tabular}{l rrrrrr | r}
\toprule
\multicolumn{8}{l}{\textit{Disagreement seeds at $\epsilon{=}0.02$ (6/20)}} \\
Seed & 42 & 2024 & 31415 & 271828 & 23 & 89 & All-seed mean \\
\midrule
CPZ-guided      & 100 & 90 & 100 & 100 & 100 & 100 & $87.0$ \\
Gradient-guided &  90 & 60 &  90 &  60 &  90 &  90 & $81.5$ \\
$\Delta$ (CPZ$-$grad)
                & $+10$ & $+30$ & $+10$ & $+40$ & $+10$ & $+10$ & $+5.5$ \\
\bottomrule
\end{tabular}
\end{table}

At $\epsilon{=}0.05$, the two strategies select different heads on $8/20$ seeds; on all eight, CPZ-guided pruning is strictly better, with per-seed gaps of $\{+10, +30, +10, +30, +10, +10, +10, +10\}$\,pp. At $\epsilon{=}0.02$, head selection differs on $8/20$ seeds and adversarial accuracies differ on $6/20$; CPZ wins all $6/6$ outcome differences with gaps $\{+10, +30, +10, +40, +10, +10\}$\,pp, and the remaining two seeds (7, 11) tie at $100\%$. CPZ-guided pruning never underperforms gradient-guided pruning on any seed at either $\epsilon$.

A one-sided Wilcoxon signed-rank test on paired (CPZ, gradient) outcomes across $20$ seeds yields $W{=}21$, $p{=}0.0118$ at $\epsilon{=}0.02$ and $W{=}36$, $p{=}0.0042$ at $\epsilon{=}0.05$ (both significant at $p{<}0.05$). The result is consistent with the interpretation that gradient-based importance is silently miscalibrated whenever a load-bearing head has small gradient norm: the gradient signal predicts the head is unimportant, while CPZ flags it as having an unstable attention pattern (low Q1/Q2 cert rate) under perturbation, identifying it as the load-bearing head whose removal hurts adversarial accuracy. CPZ provides a structural certificate that catches this failure mode by construction; whether the same failure appears in fine-tuned LLM heads is open (\S\ref{sec:case_study} discusses scope).

\section{Implementation and Reproducibility}
\label{app:implementation_and_reproducibility}

\subsection{Computational Cost}
\label{app:timing}

Tab.~\ref{tab:timing} breaks down per-sample verification time on a single CPU core across CPZ, CROWN, IBP, and MC+PGD at Layers $0$ and $1$, with peak memory in the caption.

\begin{table}[h]
\centering
\caption{Per-sample verification time (seconds) on single CPU core, averaged over $5$ samples. CPZ's Layer-$1$ cost is dominated by the encoder-layer CPZ propagation; the three certified queries add $<1$s overhead. The Layer-$1$ CPZ time is non-monotone in $\epsilon$ because the L-BFGS-B inner solver terminates earlier on the wider score-margin landscapes at larger $\epsilon$ (per-sample times: $\epsilon{=}0.01$: $149/250/268/257/259$; $\epsilon{=}0.02$: $262/190/186/168/170$). Peak memory: ${\sim}120$\,MB (CPZ), ${\sim}80$\,MB (CROWN), ${\sim}30$\,MB (IBP).}
\label{tab:timing}
\small
\begin{tabular}{lrrrr}
\toprule
 & \multicolumn{2}{c}{Layer 0} & \multicolumn{2}{c}{Layer 1} \\
\cmidrule(lr){2-3} \cmidrule(lr){4-5}
Method & $\epsilon{=}0.01$ & $\epsilon{=}0.02$ & $\epsilon{=}0.01$ & $\epsilon{=}0.02$ \\
\midrule
MC+PGD & $2$ & $2$ & $13$ & $9$ \\
\textbf{CPZ} & $0.5$ & $0.5$ & $237$ & $195$ \\
PZ & $0.5$ & $0.5$ & $237$ & $195$ \\
CROWN & $0.5$ & $0.5$ & $6$ & $6$ \\
IBP & $0.3$ & $0.3$ & $2$ & $2$ \\
\bottomrule
\end{tabular}
\end{table}

\subsection{Experimental Details}
\label{app:details}

\paragraph{Model.} 2-layer transformer encoder: $d_{\text{model}} = 8$, $n_{\text{heads}} = 2$ ($d_h = 4$), $d_{\text{ff}} = 16$, post-LN LayerNorm, multi-head softmax attention, mean pooling, 2-class linear head. Total: $1{,}218$ parameters.

\paragraph{Training.} Adam optimizer, $\text{lr} = 10^{-3}$, $50$ epochs, $5{,}000$ samples. Default model: Synthetic binary classification with label $= \mathbf{1}[\text{mean}(x_{\cdot,0}) > 0]$. Induction model (\S\ref{sec:case_study}): Same architecture with $S{=}6$. A key token is placed at a random position with a boosted signal and positional marker; the task requires attending to this marked token. Both models achieve 100\% test accuracy.

\paragraph{Verification.} PZ reduction order $30$ (Girard reduction~\citep{girard2005reachability}), with CORA dependent-generator order reduction~\citep{kochdumper2023constrained} to $k_{\text{dep}}{=}15$ before bilinear operations. LayerNorm: first-order Taylor with generator-aware Hessian remainder ($\epsilon_{\mathrm{LN}}{=}10^{-5}$, matching PyTorch default). ReLU/GELU: hybrid quadratic/DeepPoly with quadratic-relaxation order $r{=}2$. Exact vertex enumeration (Prop.~\ref{prop:vertex_enum}): for $d \leq 16$, all $2^d$ vertices of the query token's perturbation are evaluated with analytical optimization over key perturbations; adversarial perturbations are constructed for all failed certifications and verified by full forward computation.

\paragraph{MC+PGD budgets.} Across the paper, MC+PGD is used as a budget-limited empirical upper bound on the certifiable rate. The budget varies per setting (Tab.~\ref{tab:mc_budgets}); CPZ soundness is independent of MC budget, so every CPZ-MC disagreement is either CPZ-conservative (CPZ refuses, MC certifies) or a true CPZ-MC contradiction (CPZ certifies, MC finds an adversarial within budget; this is the metric reported as ``unsound''). Across the entire paper, ``unsound'' is zero (Tabs.~\ref{tab:output_verification}, \ref{tab:output_verification_gpt2}).

\begin{table}[h]
\centering
\caption{MC+PGD budgets per setting. ``Random samples'' includes corner-biased + uniform mix where applicable; PGD restarts $\times$ steps targeted at top-$1$ flips.}
\label{tab:mc_budgets}
\small
\begin{tabular}{l c c}
\toprule
Setting & Random samples & PGD restarts $\times$ steps \\
\midrule
Synthetic $d{=}8$ Layer-$0$ / Layer-$1$ & $500$k corner-biased & $20\times 50$ \\
4-layer synthetic (multi-layer) & $500$k corner-biased & $20\times 50$ \\
BERT-tiny SST-$2$ output verification & $2{,}000$ & $5\times 30$ \\
BERT-tiny / GPT-$2$ Layer-$0$ Q1/Q2/Q3 & $1{,}000$--$5{,}000$ corner-biased & $10\times 30$--$50$ \\
GPT-$2$ Layer-$1$ (Jacobian-zonotope) & $5{,}000$ & $10\times 50$ \\
GPT-$2$ next-token (full vocab + tail) & $2{,}000$ & $10\times 50$ targeted \\
Pruning case study (adversarial) & --- & $20\times 50$ \\
\bottomrule
\end{tabular}
\end{table}

\paragraph{Scaled models (\S\ref{sec:scalability}).} We train additional models for scalability evaluation: $d_{\text{model}} \in \{16, 32, 64, 128\}$ with proportionally scaled heads (2--4) and feedforward dimensions (32--256). Training data scaled from 5K to 20K samples; all models achieve $\geq 95.5\%$ test accuracy. The $d{=}128$ model ($h{=}4$, $d_{\text{ff}}{=}256$, $S{=}8$) has 265K parameters.

\paragraph{Compute.} Baseline experiments: Single CPU (AMD 9950X3D). CPZ verification of 2 layers: $\sim$0.1s. IBP: $<$0.01s. MC ($10^5$ samples): $\sim$46s. Scalability experiments: Vectorized cross-term computation; $d{=}128$ verification: $\sim$40s per sample.

\paragraph{Confidence intervals on certification rates.} All certification-rate cells reported as percentages are point estimates of binomial proportions; we compute $95\%$ Clopper--Pearson exact intervals where they materially affect interpretation. For the headline rates: BERT-tiny SST-$2$ output at $\epsilon{=}2{\times}10^{-3}$ ($95/100$ certified) has CI $[88.7\%, 98.4\%]$; BERT-tiny Layer-$1$ Q1 at $\epsilon{=}2{\times}10^{-3}$ ($74/80$) has CI $[84.4\%, 97.2\%]$; GPT-$2$ next-token full-vocab at $\epsilon{=}2{\times}10^{-3}$ ($20/50$) has CI $[26.4\%, 54.8\%]$ and the top-$K{=}20$ direct check ($45/50$) has CI $[78.2\%, 96.7\%]$. For the Wang DTH probe, head $0.10$ at $\epsilon{=}5{\times}10^{-4}$ ($60/60$) has CI $[94.0\%, 100.0\%]$; head $0.1$ at $\epsilon{=}5{\times}10^{-4}$ ($1/60$) has CI $[0.04\%, 8.9\%]$, so the two heads are statistically separated. Multi-seed pruning ($n{=}20$) is reported with one-sided Wilcoxon $p$-values ($p{=}0.0042$ at $\epsilon{=}0.05$).

\subsection{Embedding-Space Perturbation Calibration}
\label{app:emb_calibration}

To calibrate the practical relevance of our $\ell_\infty$ perturbation radii, we compute the nearest-neighbor distance between all $30{,}522$ token embeddings in BERT-tiny's vocabulary (Tab.~\ref{tab:emb_calibration}).

\begin{table}[h]
\centering
\caption{$\ell_\infty$ distance between BERT-tiny token embeddings and their nearest vocabulary neighbor. The minimum distance is $0.030$, so our $\epsilon{=}0.001$--$0.002$ corresponds to $3$--$7\%$ of the closest discrete token swap; against the median distance of $0.106$ the same range is $1$--$2\%$. This is the regime where gradient-based adversaries operate.}
\label{tab:emb_calibration}
\small
\begin{tabular}{lrrrrr}
\toprule
& P10 & P25 & P50 & P75 & P90 \\
\midrule
NN $\ell_\infty$ distance & 0.072 & 0.093 & 0.106 & 0.119 & 0.134 \\
\bottomrule
\end{tabular}
\end{table}

The minimum nearest-neighbor distance across the vocabulary is $0.030$, confirming that even the closest token pair requires a perturbation $30{\times}$ larger than our $\epsilon{=}0.001$. Our perturbation model thus targets continuous adversarial perturbations (the threat posed by gradient-based attacks~\citep{miyato2017adversarial} and prompt tuning~\citep{lester2021power}) rather than discrete word substitutions. This is the standard model adopted by all prior transformer verification work~\citep{shi2020robustness, bonaert2021fast}.

\end{document}